\documentclass{article}
\usepackage[letterpaper, left=1in, right=1in, top=1in, bottom=1in]{geometry}
\usepackage{natbib}
\bibpunct{(}{)}{;}{a}{,}{,}
\usepackage[colorlinks,citecolor=blue!70!black,linkcolor=blue!70!black,
urlcolor=blue!70!black,breaklinks=true]{hyperref}
\usepackage{parskip}
\usepackage[dvipsnames]{xcolor}
\usepackage{microtype}
\usepackage{booktabs} 
\usepackage{amsmath}
\usepackage{amsthm}
\usepackage{amssymb}
\usepackage{dsfont} 
\usepackage{xspace}
\usepackage{url}            
\usepackage{algorithm}
\usepackage[noend]{algorithmic}
\usepackage{listings}
\usepackage{bm}
\usepackage{bbm}
\usepackage{enumitem}
\usepackage{nicefrac}       
\usepackage{mathrsfs} 
\usepackage{inconsolata}
\usepackage{wrapfig}
\usepackage{graphicx}
\usepackage{bigints}
\usepackage{transparent}
\usepackage{comment}
\usepackage{thmtools}
\usepackage[nameinlink,capitalize]{cleveref}
\makeatletter
\AddToHook{cmd/appendix/before}{
\def\cref@section@alias{appendix}
\def\cref@subsection@alias{appendix}
\def\cref@subsubsection@alias{appendix}
}
\makeatother
\usepackage{thm-restate}
\usepackage{nicematrix}
\usepackage{arydshln}
\usepackage{fancyhdr}
\usepackage{mathtools}
\usepackage[scaled=.88]{helvet}
\usepackage{breakcites}
\usepackage{multirow}
\usepackage{subcaption}
\usepackage{caption}
\usepackage[most]{tcolorbox}

\DeclareMathOperator*{\argmax}{arg\,max}
\DeclareMathOperator*{\argmin}{arg\,min}

\DeclareMathOperator{\unif}{{unif}}

\renewcommand{\epsilon}{\varepsilon}

\newtheoremstyle{spaced}
  {6pt}   %
  {0pt}   %
  {\itshape} %
  {}       %
  {\bfseries} %
  {.}      %
  {0.5em}  %
  {}
\theoremstyle{spaced}

\newtheorem{assumption}{Assumption}

\newcommand{\algcommentlight}[1]{\textcolor{blue!70!black}{\transparent{0.5}\small{\texttt{\textbf{//\hspace{2pt}#1}}}}}

\DeclarePairedDelimiter{\abs}{\lvert}{\rvert} %

\DeclarePairedDelimiter{\crl}{\{}{\}}
\DeclarePairedDelimiter{\prn}{(}{)}

\DeclarePairedDelimiterX{\infdiv}[2]{(}{)}{%
  #1\;\delimsize\|\;#2%
}

\newcommand{\wt}[1]{\widetilde{#1}}
\newcommand{\wh}[1]{\widehat{#1}}
\newcommand{\wb}[1]{\widebar{#1}}

\def\ddefloop#1{\ifx\ddefloop#1\else\ddef{#1}\expandafter\ddefloop\fi}
\def\ddef#1{\expandafter\def\csname bb#1\endcsname{\ensuremath{\mathbb{#1}}}}
\ddefloop ABCDEFGHIJKLMNOPQRSTUVWXYZ\ddefloop
\def\ddefloop#1{\ifx\ddefloop#1\else\ddef{#1}\expandafter\ddefloop\fi}
\def\ddef#1{\expandafter\def\csname b#1\endcsname{\ensuremath{\mathbf{#1}}}}
\ddefloop ABCDEFGHIJKLMNOPQRSTUVWXYZ\ddefloop
\def\ddef#1{\expandafter\def\csname sf#1\endcsname{\ensuremath{\mathsf{#1}}}}
\ddefloop ABCDEFGHIJKLMNOPQRSTUVWXYZ\ddefloop
\def\ddef#1{\expandafter\def\csname c#1\endcsname{\ensuremath{\mathcal{#1}}}}
\ddefloop ABCDEFGHIJKLMNOPQRSTUVWXYZ\ddefloop
\def\ddef#1{\expandafter\def\csname h#1\endcsname{\ensuremath{\widehat{#1}}}}
\ddefloop ABCDEFGHIJKLMNOPQRSTUVWXYZ\ddefloop
\def\ddef#1{\expandafter\def\csname hc#1\endcsname{\ensuremath{\widehat{\mathcal{#1}}}}}
\ddefloop ABCDEFGHIJKLMNOPQRSTUVWXYZ\ddefloop
\def\ddef#1{\expandafter\def\csname t#1\endcsname{\ensuremath{\widetilde{#1}}}}
\ddefloop ABCDEFGHIJKLMNOPQRSTUVWXYZ\ddefloop
\def\ddef#1{\expandafter\def\csname tc#1\endcsname{\ensuremath{\widetilde{\mathcal{#1}}}}}
\ddefloop ABCDEFGHIJKLMNOPQRSTUVWXYZ\ddefloop
\def\ddefloop#1{\ifx\ddefloop#1\else\ddef{#1}\expandafter\ddefloop\fi}
\def\ddef#1{\expandafter\def\csname scr#1\endcsname{\ensuremath{\mathscr{#1}}}}
\ddefloop ABCDEFGHIJKLMNOPQRSTUVWXYZ\ddefloop

\let\oldparagraph\paragraph
\renewcommand{\paragraph}[1]{\oldparagraph{#1}}

\renewcommand{\epsilon}{\varepsilon}

\newcommand{\ldef}{\vcentcolon=}

\renewcommand{\bigm}[1]{%
  \ifcsname fenced@\string#1\endcsname
    \expandafter\@firstoftwo
  \else
    \expandafter\@secondoftwo
  \fi
  {\expandafter\amsmath@bigm\csname fenced@\string#1\endcsname}%
  {\amsmath@bigm#1}%
}

\newcommand{\DeclareFence}[2]{\@namedef{fenced@\string#1}{#2}}
\makeatother

\makeatletter
\let\save@mathaccent\mathaccent
\newcommand*\if@single[3]{%
  \setbox0\hbox{${\mathaccent"0362{#1}}^H$}%
  \setbox2\hbox{${\mathaccent"0362{\kern0pt#1}}^H$}%
  \ifdim\ht0=\ht2 #3\else #2\fi
  }
\newcommand*\rel@kern[1]{\kern#1\dimexpr\macc@kerna}
\newcommand*\widebar[1]{\@ifnextchar^{{\wide@bar{#1}{0}}}{\wide@bar{#1}{1}}}
\newcommand*\wide@bar[2]{\if@single{#1}{\wide@bar@{#1}{#2}{1}}{\wide@bar@{#1}{#2}{2}}}
\newcommand*\wide@bar@[3]{%
  \begingroup
  \def\mathaccent##1##2{%
    \let\mathaccent\save@mathaccent
    \if#32 \let\macc@nucleus\first@char \fi
    \setbox\z@\hbox{$\macc@style{\macc@nucleus}_{}$}%
    \setbox\tw@\hbox{$\macc@style{\macc@nucleus}{}_{}$}%
    \dimen@\wd\tw@
    \advance\dimen@-\wd\z@
    \divide\dimen@ 3
    \@tempdima\wd\tw@
    \advance\@tempdima-\scriptspace
    \divide\@tempdima 10
    \advance\dimen@-\@tempdima
    \ifdim\dimen@>\z@ \dimen@0pt\fi
    \rel@kern{0.6}\kern-\dimen@
    \if#31
      \overline{\rel@kern{-0.6}\kern\dimen@\macc@nucleus\rel@kern{0.4}\kern\dimen@}%
      \advance\dimen@0.4\dimexpr\macc@kerna
      \let\final@kern#2%
      \ifdim\dimen@<\z@ \let\final@kern1\fi
      \if\final@kern1 \kern-\dimen@\fi
    \else
      \overline{\rel@kern{-0.6}\kern\dimen@#1}%
    \fi
  }%
  \macc@depth\@ne
  \let\math@bgroup\@empty \let\math@egroup\macc@set@skewchar
  \mathsurround\z@ \frozen@everymath{\mathgroup\macc@group\relax}%
  \macc@set@skewchar\relax
  \let\mathaccentV\macc@nested@a
  \if#31
    \macc@nested@a\relax111{#1}%
  \else
    \def\gobble@till@marker##1\endmarker{}%
    \futurelet\first@char\gobble@till@marker#1\endmarker
    \ifcat\noexpand\first@char A\else
      \def\first@char{}%
    \fi
    \macc@nested@a\relax111{\first@char}%
  \fi
  \endgroup
}
\makeatother

\newcommand{\oursmath}{\mathsf{ours}}
\newcommand{\unifmath}{\mathsf{unif}}
\newcommand{\etcmath}{\mathsf{ETC}}
\newcommand{\metric}{\mathsf{Metric}}
\newcommand{\precisionmath}{\mathsf{Accuracy}}
\newcommand{\precisiontext}{\text{accuracy}\xspace}
\newcommand{\coveragetext}{\text{coverage}\xspace}
\newcommand{\coveragemath}{\mathsf{Coverage}}
\newcommand{\uniftext}{\textsc{Uniform}\xspace}
\newcommand{\elimtext}{\textsc{Elimination}\xspace}
\newcommand{\ucbtext}{\textsc{UCB}\xspace}
\newcommand{\gaptext}{\textsc{Gap}\xspace}
\newcommand{\entropytext}{\textsc{Entropy}\xspace}

\title{Strategic Scaling of Test-Time Compute:\\ A Bandit Learning Approach}
\date{}
\author{
Bowen Zuo\\
{\normalsize University of California, Riverside}\\
{\normalsize\texttt{bzuo002@ucr.edu}}
\and
Yinglun Zhu\textsuperscript{\dag}\\
{\normalsize University of California, Riverside}\\
{\normalsize\texttt{yzhu@ucr.edu}}
}
\begin{document}

\maketitle
\begingroup
\renewcommand\thefootnote{}\footnotetext{\textsuperscript{\dag}Project lead and corresponding author.}
\endgroup

\begin{abstract}
  Scaling test-time compute has emerged as an effective strategy for improving the performance of large language models. However, existing methods typically allocate compute uniformly across all queries, overlooking variation in query difficulty. To address this inefficiency, we formulate test-time compute allocation as a novel bandit learning problem and propose adaptive algorithms that estimate query difficulty on the fly and allocate compute accordingly. Compared to uniform allocation, our algorithms allocate more compute to challenging queries while maintaining accuracy on easier ones. Among challenging queries, our algorithms further learn to prioritize solvable instances, effectively reducing excessive computing on unsolvable queries. We theoretically prove that our algorithms achieve better compute efficiency than uniform allocation and empirically validate their effectiveness on math and code benchmarks. Specifically, our algorithms achieve up to an 11.10\% performance improvement (15.04\% relative) on the MATH-500 dataset, up to 10.82\% (14.44\% relative) on the AIME25 dataset, and up to an 11.23\% performance improvement (15.29\% relative) on the LiveCodeBench dataset.

\end{abstract}

\section{Introduction}
\label{sec:intro}

Recent advances in large language models (LLMs) have shifted attention from training-time compute \citep{kaplan2020scalinglawsneurallanguage, hoffmann2022training, chowdhery2022palmscalinglanguagemodeling} to test-time compute \citep{wei2023chainofthoughtpromptingelicitsreasoning, yao2023treethoughtsdeliberateproblem, madaan2023self, agarwal2024manyshotincontextlearning, muennighoff2025s1simpletesttimescaling} as a means of improving model performance.  
Test-time scaling methods such as Best-of-$N$ sampling \citep{brown2024largelanguagemonkeysscaling, snell2024scalingllmtesttimecompute} and consistency checking \citep{wang2022self} enhance output quality by generating multiple responses and selecting the most promising one.  
This selection process can be strengthened using high-quality reward oracles \citep{cobbe2021trainingverifierssolvemath, uesato2022solving, lightman2023letsverifystepstep, zhang2025lessons}.  
These methods have achieved strong empirical gains without additional model training.  
For instance, as noted in OpenAI’s o1 release report \citep{openai2024learning}, repeated sampling with 64 generations improves accuracy on the 2024 AIME competition math dataset from 74.4\% to 83.3\%---a nearly 9\% gain without any model updates.
\looseness=-1

Despite recent advances, most test-time scaling techniques still allocate compute \emph{uniformly across all queries} \citep{brown2024largelanguagemonkeysscaling, snell2024scalingllmtesttimecompute}, ignoring the inherent variability in query difficulty.  
This one-size-fits-all strategy is inefficient: simple arithmetic questions receive the same compute as multi-step reasoning tasks, leading to wasted resources on easy queries and insufficient budget on hard ones.  
Ideally, one should allocate \emph{just enough compute} to confidently solve easy queries and \emph{reallocate the remaining budget} to harder ones.  
While recent work has begun exploring adaptive test-time strategies, existing methods either (1) focus on compute allocation \emph{within a single query} \citep{sun2024fast, manvi2024adaptiveinferencetimecomputellms, tan2025adaptiverectificationsamplingtesttime}, or (2) rely on \emph{two-stage procedures} \citep{damani2024learning, wang2025makepennycountdifficultyadaptive} that require training an auxiliary model (or pre-compute allocation) in the first stage to guide later allocation decisions.

In this work, we introduce a new perspective---\emph{strategic scaling of test-time compute}---in which compute is adaptively allocated \emph{across} a set of queries based on their estimated difficulty.  
We formulate test time compute allocation as a \emph{fully adaptive} pure-exploration-style bandit learning problem \citep{bubeck2009pure, jamieson2014best, locatelli2016optimalalgorithmthresholdingbandit, zhu2020robust}, treating each query as an action and allocating compute sequentially to \emph{maximize the number of queries answered correctly within a fixed budget}.  
Our adaptive algorithms estimate query difficulty on the fly and prioritize compute for those most likely to benefit from additional inference.  
Empirically, our method achieves up to 11.10\% absolute (15.04\% relative) improvement on the MATH-500 dataset \citep{lightman2023letsverifystepstep, hendrycksmath2021}, 10.82\% absolute (14.44\% relative) on AIME25 \citep{aime2025}, and 11.23\% absolute (15.29\% relative) on LiveCodeBench \citep{jain2024livecodebenchholisticcontaminationfree}—all under the same compute budget as baselines.  
\cref{fig:illustration} provides a high-level comparison between our method and various baselines.

\begin{figure}[t]
  \centering
  \includegraphics[width=.25\linewidth]{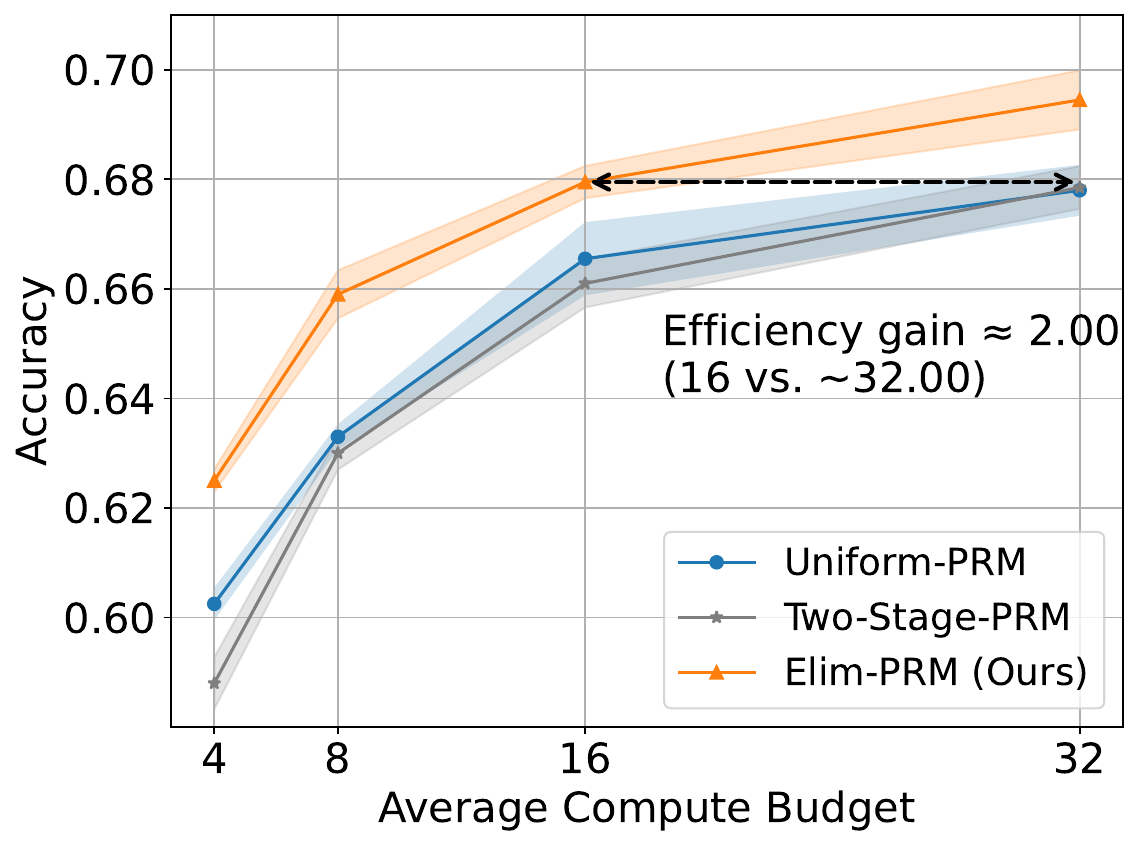}\hfill
   \includegraphics[width=.25\linewidth]{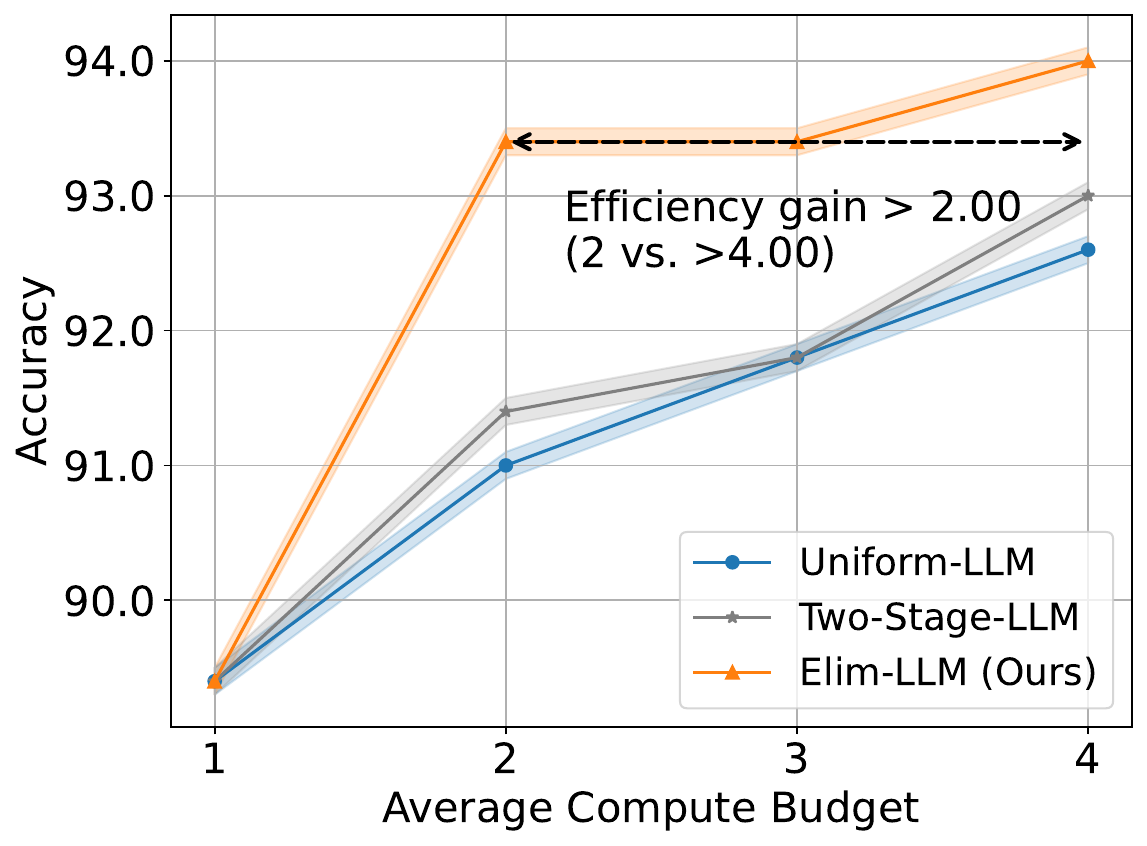}\hfill
  \includegraphics[width=.25\linewidth]{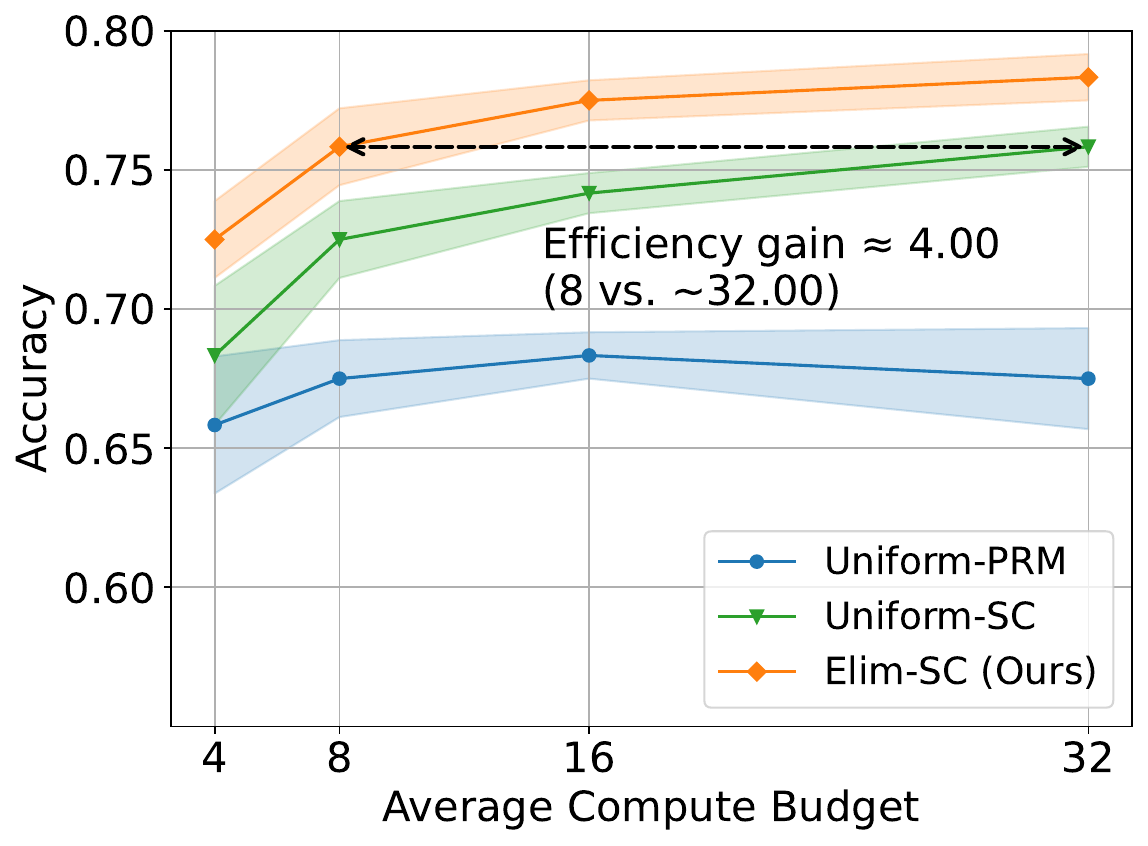}\hfill
  \includegraphics[width=.25\linewidth]{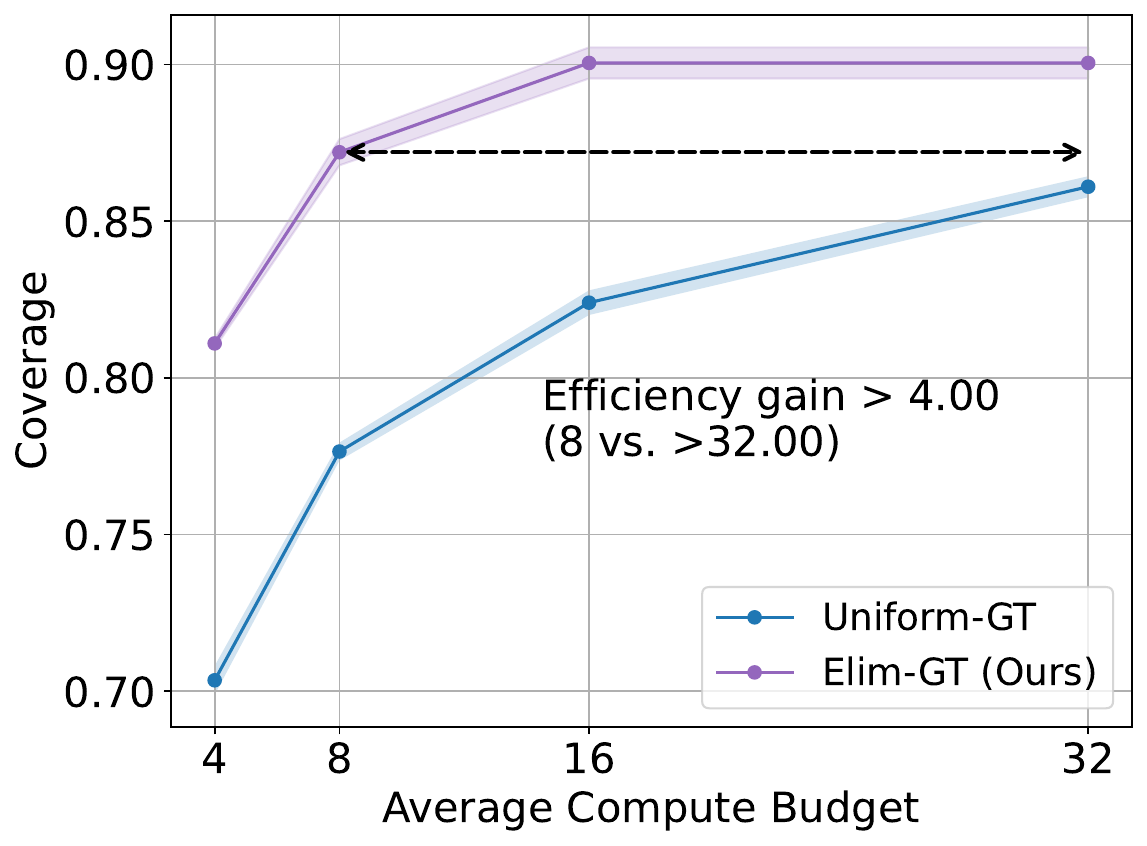}\hfill
  \caption{Comparison between our algorithm and baselines. 
  \emph{First:} Accuracy comparison on MATH-500 with \texttt{Llama-3.1-8B-Instruct}.  
  \emph{Second:} Accuracy comparison on MATH-500 with \texttt{Gemini-2.5-flash-lite}.
  \emph{Third:} Accuracy comparison on AIME25 with \texttt{Qwen3-4B}.
  \emph{Fourth:} Coverage comparison on LiveCodeBench with \texttt{DeepSeek-R1-Distill-Llama-8B}.  
  }
  \label{fig:illustration}
\end{figure}

\paragraph{Contributions.} 
We summarize our main contributions below:
\begin{enumerate}
[leftmargin=10pt, itemindent=*]
  \item We formulate LLM test-time compute allocation as a novel bandit learning problem, bridging test-time scaling and bandit learning communities.
  This formulation grounds strategic test-time scaling in a precise decision-theoretic framework.
  \item 
  We propose a general algorithmic framework for strategic compute allocation, supporting flexible exploration strategies—including a novel entropy-based rule.  
Our framework naturally extends to incorporate alternative aggregation methods and handles both streaming and token-constrained settings.  
We further provide theoretical insights into the efficiency gains of our adaptive approach.

    \item We conduct extensive experiments on math and code benchmarks and show that our algorithms consistently outperform baselines. Further analyses demonstrate that our algorithms adaptively allocate compute to harder queries in standard settings, and to solvable queries in scenarios containing both solvable and unsolvable instances, effectively avoiding compute waste.
\end{enumerate}

\textbf{Paper organization.}
 The rest of this paper is organized as follows. 
 We review related work in \cref{sec:related} and 
 introduce the problem of strategic test-time compute allocation in \cref{sec:setting}.
Our solution is presented in \cref{sec:methods}, including the bandit formulation, algorithmic framework, extensions, and theoretical analysis.
Empirical results are in \cref{sec:experiments}, covering main results, analyses, and ablations.
We conclude in \cref{sec:discussion}.
 Complete proofs and additional experimental details are deferred to the Appendix.

\section{Related work}
\label{sec:related}

\paragraph{Test-time compute techniques.}
Scaling test-time compute (TTC) has emerged as a powerful class of methods for improving the performance of large language models, typically without requiring additional parameter updates.  
In-context learning \citep{brown2020languagemodelsfewshotlearners}, including its scaling to many-shot regimes \citep{agarwal2024manyshotincontextlearning, bertsch2024incontextlearninglongcontextmodels}, as well as prompting- or search-based methods such as Chain-of-Thought \citep{wei2023chainofthoughtpromptingelicitsreasoning} and Tree-of-Thought \citep{yao2023treethoughtsdeliberateproblem, feng2023alphazero}, have demonstrated that carefully designed test-time techniques can match or even surpass finetuned models \citep{mosbach2023fewshotfinetuningvsincontext}.  
Self-reflection \citep{madaan2023self} is another popular technique for leveraging TTC to improve performance: by prompting the LLM to iteratively refine its own generations, the model can produce higher-quality responses across a range of tasks \citep{chen2023teaching, gou2023critic}.  
\citet{muennighoff2025s1simpletesttimescaling} further demonstrates that simply increasing the number of generated ``thinking'' tokens leads to substantial performance gains.
\looseness=-1

Repeated sampling methods---most notably Best-of-$N$ \citep{brown2024largelanguagemonkeysscaling, snell2024scalingllmtesttimecompute, wang2022self}---have become popular for scaling test-time compute, especially when combined with high-quality reward models \citep{cobbe2021trainingverifierssolvemath, uesato2022solving, lightman2023letsverifystepstep, zhang2025lessons}.  
Building on this line of work, recent---and in some cases concurrent---efforts have proposed adaptive variants of Best-of-$N$ that dynamically allocate compute for a given query \citep{sun2024fast, manvi2024adaptiveinferencetimecomputellms, tan2025adaptiverectificationsamplingtesttime}.  
However, these methods focus on adaptive allocation \emph{within an individual query}, without considering opportunities to redistribute compute across a set of queries.  
In contrast, we study \emph{strategic compute allocation across multiple queries}, introducing an additional layer of optimization---for example, deciding when to transfer unused budget from easier queries to harder ones.  
The problem setting in \citet{damani2024learning} is closely related, as they also consider multi-query compute allocation.  
However, their approach relies on a two-stage schedule that requires training an additional model in the first stage to guide compute distribution, incurring extra compute cost.  
In contrast, we formulate the problem as a novel bandit learning task and develop \emph{fully adaptive} algorithms that learn to allocate compute \emph{on the fly}, without any additional training overhead.  
Furthermore, we provide the first theoretical result that provably demonstrates the advantage of strategic test-time compute allocation over uniform allocation.

\paragraph{Bandit learning and pure exploration.} 
Bandit learning is a fundamental framework for sequential decision making under uncertainty, where an agent must choose among a set of actions (or arms) to optimize a long-term objective with limited feedback \citep{bubeck2012regretanalysisstochasticnonstochastic, lattimore2020bandit}.
Popular algorithms include Upper Confidence Bound (UCB,  \citet{auer2002finite, audibert2009minimax, chu2011contextual, zhu2020regret, zhu2022pareto, garivier2022kl}), which selects the action with the highest upper confidence bound; Thompson Sampling \citep{thompson1933likelihood, chapelle2011empirical, agrawal2012analysis, russo2018tutorial}, which selects the action with the highest sampled reward from the posterior; and inverse gap weighting strategies \citep{foster2020beyond, foster2021statistical, zhu2022large, zhu2022contextual, rucker2023infinite},
which sample actions with probabilities inversely proportional to their estimated reward gaps.
Bandit algorithms have been widely applied in domains such as online recommendation systems \citep{li2010contextual}, clinical trials \citep{villar2015multi}, hyperparameter tuning \citep{li2018hyperband}, and more recently applications with LLMs \citep{shi2024efficient, chen2024efficient}.

Pure exploration \citep{bubeck2009pure, jamieson2014best}, also known as the best arm identification (BAI) problem, is a key subfield of bandit learning that aims to identify high-performing arms using as few samples as possible.
Core algorithms include successive elimination \citep{even2002pac, even2006action, karnin2013almost}, UCB-based strategies \citep{kalyanakrishnan2012pac, kaufmann2013information, jamieson2014lil}, and gap-based sampling methods \citep{locatelli2016optimalalgorithmthresholdingbandit}.
Recent extensions generalize these techniques to more expressive function classes, including linear models \citep{fiez2019sequential, katz2020empirical, zhu2022near}, kernel functions \citep{du2021collaborative}, and neural networks \citep{zhu2021pure}.
Our work introduces a novel pure-exploration-style bandit formulation, tailored to LLM test-time compute allocation---a setting not previously explored in this context.
We treat each query as a bandit action and adaptively allocate compute to maximize the fraction of queries correctly answered under a fixed compute budget.
This formulation enables the use of classical bandit techniques such as elimination rules, confidence bounds, and gap-based sampling. In addition, we propose a new entropy-based sampling strategy (\cref{sec:extensions}) 
that prioritizes queries with diverse response patterns.
While our formulation is conceptually related to the thresholding bandit problem and its variants
\citep{locatelli2016optimalalgorithmthresholdingbandit, zhu2020regret}, it departs fundamentally in its objective.
Thresholding bandits aim to 
identify actions (queries) whose \emph{expected reward} exceeding a given threshold.
In contrast, our goal is to generate at least one high-quality response for each query, 
regardless of its expected score.

\vspace{-5 pt}
\section{Problem setting}
\label{sec:setting}

Let $p$ denote a language model, which takes a query $x \in \cX$ as input and generates a response $y \sim p(\cdot \mid x)$. 
Recent studies show that scaling up the test-time compute can significantly improve the performance of LLMs across a variety of tasks \citep{snell2024scalingllmtesttimecompute}. 
In this context, we consider the amount of test-time compute as the total number of responses generated by the language model. 
For example, given query $x\in \cX$ and a compute budget of $N$, the model can generate a set of $N$ responses $g(x; N) \ldef \crl{y_1, \cdots, y_N}$, where each response $y_i \sim p(\cdot \mid x)$ is sampled from the conditional distribution $p(\cdot \mid x)$.
A \emph{reward oracle}  $r: \cX \times \cY \rightarrow [0, 1]$ is used to evaluate the quality of each generation; the reward oracle can be instantiated by either a ground truth verifier or a learned reward model \citep{cobbe2021trainingverifierssolvemath, uesato2022solving,lightman2023letsverifystepstep, zhang2025lessons}.
When the evaluation metric requires a single response as the output, test-time compute methods such as the Best-of-$N$ algorithm \citep{brown2024largelanguagemonkeysscaling} use the reward oracle to score each response and return the one with the highest score. 
Specifically, given a set of responses $g(x; N) = \crl{y_1, \cdots, y_N}$ and letting  $r(x, y_i)$ denote the score of response $y_i$, the final output $f (x;N) \ldef f(g(x; N))$ is defined as:
\looseness=-1
 \begin{align*}
     f(x;N) = y_{i^\star}, \quad \text{where} \quad i^\star \ldef \argmax_{i \in [N]} r(x,y_i).
 \end{align*}

 While scaling test-time compute can improve performance, existing methods primarily focus on \emph{uniform allocation of compute budget}.
 Specifically, given a set of queries $S = \crl{x_1, \cdots, x_n}$ and total compute budget $B \ldef n \wb B$, existing approaches assign the same compute budget $\wb B$ to each query $x_i$ and generate the final outputs $\crl{\prn{x_1, f(x_1; \wb B)}, \cdots, \prn{x_n, f(x_n; \wb B)}}$.
 This uniform allocation is inefficient: it ignores differences in query difficulty and \emph{assigns the same compute to both easy and hard queries}. \looseness=-1
 
 \subsection{Strategic test-time compute allocation}
 \label{sec:new_problem}
 To address the limitations of uniform allocation, we study the problem of \emph{strategic test-time compute allocation}---how to \emph{adaptively} allocate a total compute budget across a set of queries to \emph{maximize the fraction of correctly answered queries}.  
Let $B$ denote the total compute budget and $S = \crl{x_1, \cdots, x_n}$ be a set of $n$ queries.
Let $\metric \in [0, 1]$ be an evaluation metric and $c(x_i)$ be the compute allocated to query $x_i$. 
The goal is to maximize the overall performance subject to a budget constraint:
\begin{align}
    \label{eq:allocation}
    \max_{{\crl{c(x_i)}}_{i=1}^n} \, \frac{1}{n} \sum_{i=1}^n \metric \prn[\big]{x_i; c(x_i)} \quad \text{subject to} \quad \sum_{i=1}^n c(x_i) \leq B.
\end{align}

We consider two popular evaluation metrics: \coveragetext and \precisiontext.
Given a compute allocation $c(x_i)$, \emph{coverage} evaluates whether any of the $c(x_i)$ generations in $g(x_i; c(x_i))$ correctly answers the query $x_i$, while \emph{accuracy} evaluates whether the final output $f(x_i; c(x_i))$ is correct.  
These metrics are defined as:
\begin{align*}
  \coveragemath(x_i; c(x_i)) & \ldef \bbI\crl{\text{there exists $y \in g(x_i; c(x_i))$ that correctly answers query $x_i$.}}\\
  \precisionmath(x_i; c(x_i)) & \ldef \bbI\crl{\text{$f(x_i; c(x_i))$ correctly answers query $x_i$}.}
\end{align*}

The key challenge in \cref{eq:allocation} is to adaptively allocate compute budget $c(x_i)$ to each query $x_i$ \emph{under uncertainty}---that is, 
without knowing in advance the difficulty of each query or how much compute is needed to answer it correctly.
To isolate and address this challenge, we adopt the standard Best-of-$N$ approach \citep{brown2024largelanguagemonkeysscaling, snell2024scalingllmtesttimecompute}  
for both compute counting (i.e., measuring the number of generations per query) and final output selection.

\section{Methods}\label{sec:methods}

We present our approaches to solve the strategic test-time compute allocation problem introduced in \cref{sec:new_problem}.
In \cref{sec:bandit}, we first formulate test-time compute allocation as a bandit learning problem.
We then introduce our algorithmic framework in \cref{subsec:algorithm}, followed by extensions in 
\cref{sec:extensions}
and theoretical analysis of compute efficiency in \cref{sec:theoretical_analysis}.

\subsection{Test-time scaling as bandit learning}
\label{sec:bandit}

To address the challenge of strategic compute allocation under uncertainty, we introduce a novel bandit learning formulation tailored to LLM test-time compute objectives.
Following the bandit terminology, we treat each query $x \in \cS$ as an \emph{action}, and interpret sampling action $x$ as allocating one unit of compute to query $x$ to obtain a randomly generated response $y$. 
After taking action $x$, the learner receives 
feedback from a reward oracle in the form of a score $r(x, y)$.

Our objective is to design an adaptive compute allocation algorithm that maximizes the fraction of queries that are correctly answered within a fixed compute budget $B$.
Assuming availability of a sufficiently accurate reward oracle (e.g., ground truth labels), 
we approximate the correctness of a response using a user-specified threshold $\gamma \in [0,1]$: a response $y$ to query $x$ is considered correct if $r(x,y) \geq \gamma$.\footnote{
We assume access to a sufficiently accurate reward oracle in order to focus on the key challenge of adaptive compute allocation. 
This assumption is clearly satisfied in settings with ground truth labels, and is approximately satisfied by recently developed process reward models \citep{zhang2025lessons}. 
}
Formally, the algorithm adaptively distributes the total compute budget $B$ across all queries through an allocation 
$\crl{c(x_i)}_{i=1}^n$, optimizing the following objective: 
\begin{align*}
  \max_{\crl{c(x_i)}_{i=1}^n} \frac{1}{n} \sum_{i=1}^n  \bbI \prn[\Big]{\max_{y \in g(x_i; c(x_i))} r(x_i, y) \geq \gamma }, 
\end{align*}
where $g(x_i; c(x_i))$ denotes the set of $c(x_i)$ responses generated for query $x_i$.\footnote{When the evaluation metric is $\precisionmath$, one must further explicitly select and output the correct response.} 

While our formulation is conceptually related to the bandit pure exploration problem \citep{bubeck2009pure, jamieson2014best} and its thresholding bandit variants \citep{locatelli2016optimalalgorithmthresholdingbandit, zhu2020robust},  
it fundamentally departs from the conventional objectives.
Standard pure exploration settings aim to identify actions (queries) with high \emph{expected} scores, which correspond---in our setting---to identifying a subset of easy queries that can be reliably answered by the LLM.
In contrast, our objective aims at generating at least one high-quality (correct) response for each query, regardless of its expected score.  
To our knowledge, this not only introduces a novel bandit formulation but also opens the door to further exploration of bandit-based LLM test-time compute allocation.

\subsection{Our algorithmic framework}
\label{subsec:algorithm}

Based on the bandit formulation, we next present our algorithmic framework in \cref{alg:methods}.
Given a query set $\cS$, \cref{alg:methods} initializes an \emph{active set} $\cA = \cS$ that contains active queries that have not yet been confidently answered.
For each query $x \in \cS$, it maintains a response set $g(x)$, the best-scoring response $\check{y}(x)$ observed so far, and its corresponding reward score $\check{r}(x)$, as evaluated by the reward oracle $r$.  
\cref{alg:methods} proceeds in rounds, and operates based on two key components: an \emph{exploration rule} and an \emph{elimination rule}:
\begin{itemize}
[leftmargin=10pt, itemindent=*]
    \item \textbf{The exploration rule.}
    At each round, \cref{alg:methods} explores all queries in the active set, i.e., for each active query $x \in \cA$, it generates $K$ new responses $\crl{y_i}_{i=1}^K$ and updates the response set $g(x) \gets g(x) \cup \crl{y_i}_{i=1}^K$. 
    We discuss extensions to this simple exploration rule in \cref{sec:extensions}.
    \item \textbf{The elimination rule.}
    For each explored query $x$, let $y_{i^\star}$ denote the response that achieves the highest score among newly generated responses, i.e., $i^\star = \argmax_{i \in [K]}r(x, y_i)$. If the reward $r(x, y_{i^\star})$ is greater than the previously observed best score $\check r(x)$, then \cref{alg:methods} (1) updates its maintained best-scoring response $\check y(x) = y_{i^\star}$ and the corresponding reward $\check r(x) = r(x, y_{i^\star})$;
    and (2) \emph{eliminates} query $x$ from the active set $\cA$ if the score $r(x, y_{i^\star})$ is also greater or equal to the elimination threshold $\gamma$.
    \looseness=-1
\end{itemize}
\cref{alg:methods} terminates when the compute budget is exhausted (i.e., $B = 0$) or when all queries have been eliminated from the active set (i.e., ${\cA}= \emptyset$).
For each query $x \in \cS$, \cref{alg:methods} outputs its maintained response set $g(x)$ for coverage evaluation, and its best-scoring response $\check y(x)$ for accuracy evaluation.
\looseness=-1

   \begin{algorithm}[t]
\caption{Strategic Test-Time Compute Allocation}
\begin{algorithmic}[1]
\label{alg:methods}
	\renewcommand{\algorithmicrequire}{\textbf{Input:}}	\renewcommand{\algorithmicensure}{\textbf{Output:}}
\REQUIRE Query set $\cS$, total compute budget $B$, reward oracle $r$, per-round per-query compute budget $K$, and elimination threshold $\gamma$.
\STATE 
For each query $x \in \cS$, 
  maintain a response set $g(x)$, the best-scoring response $\check y(x)$, and its associated reward $\check r(x)$.
\STATE Initialize the active set $\cA \gets \cS$ to be the full query set.
  \WHILE{$B > 0$ and $\abs{\cA} > 0$ }
\FOR{$x \in \cA$}  \label{line:elim1}
\STATE 
\label{line:elim2}
Generate $K$ new responses $\crl{y_i}_{i=1}^{K}$. 
  Update $g(x) \gets g(x) \cup \crl{y_i}_{i=1}^{K}$ and $B \gets B - K$.
\hfill \algcommentlight{The exploration rule: allocating compute to all queries in the active set $\cA$. We discuss extensions of the exploration rule in \cref{sec:extensions}.}
  \STATE Get $i^\star \gets \argmax_{i \in [K]} r(x, y_i)$.
\IF {$r(x, y_{i^\star}) > \check r(x)$}
\STATE Update $\check y(x) \gets y_{i^\star}$ and $\check r(x) \gets r(x, y_{i^\star})$.
  \label{line:selection}
\ENDIF
  \IF{$r(x,y_{i^\star}) \geq \gamma$}
\STATE Update $\cA \gets \cA \setminus \crl{x}$. 
  \label{line:elimination}
\hfill \algcommentlight{The elimination rule.}
\ENDIF
\ENDFOR
\ENDWHILE
  \ENSURE For each $x \in \cS$, output its response set $g(x)$ and the best-scoring response $\check y(x)$.
\algcommentlight{Use $g(x)$ for coverage evaluation and $\check y(x)$ for accuracy evaluation.}
\end{algorithmic}
\end{algorithm}

\paragraph{Reward oracles.}
Reward oracles have become a core component in test-time compute techniques, even for the vanilla uniform Best-of-$N$ algorithm \citep{brown2024largelanguagemonkeysscaling, snell2024scalingllmtesttimecompute}.
Common reward oracles include outcome reward models (ORMs, \citet{cobbe2021trainingverifierssolvemath}) and process reward models (PRMs, \citet{uesato2022solving, lightman2023letsverifystepstep, zhang2025lessons}). 
For tasks with easy or automatic verification, such as math and code generation, ground truth (GT) labels can serve as an exact reward oracle.
We emphasize that \cref{alg:methods} uses \emph{the same number of reward oracle calls} as the uniform Best-of-$N$ algorithm, which relies on the reward oracle to select the final output.

\paragraph{Hyperparameters.}
\cref{alg:methods} takes two hyperparameters as input: the per-round per-query compute budget $K$ and a user-specified elimination threshold $\gamma$. 
The hyperparameter per-round per-query compute budget $K$ controls the granularity level of the budget allocation: a smaller value of $K$ leads to more fine-grained budget allocation with an increased number of allocation rounds. 
The elimination hyperparameter $\gamma $ decides when to eliminate a query from the active set $\cA$. The value of $\gamma$ can be determined based on expert knowledge or based on cross-validation on a separate training set. 
These hyperparameters offer additional levels of flexibility for \cref{alg:methods}. 
We conduct ablation studies of these hyperparameters in \cref{subsec:ablation} and  \cref{app:thresholds}.

\subsection{Extensions of \cref{alg:methods}}
\label{sec:extensions}

\paragraph{\cref{alg:methods} with different aggregation strategies.}
While our main discussion centers on Best-of-$N$, the proposed framework is flexible and can accommodate alternative aggregation strategies.
Prior work \citep{wang2025thinkdeepthinkfast} has shown that Self-Consistency (SC) is often more effective for reasoning models—such as \texttt{Qwen3-4B}—due to their tendency to produce logically coherent outputs.
To incorporate SC into \cref{alg:methods}, we make two modifications: (1) the selection rule (line~\ref{line:selection}) now uses SC instead of a reward model (PRM), and (2) the elimination rule (line~\ref{line:elimination}) is updated to eliminate a query once a certain proportion of its collected responses converge to the same answer (e.g., when over 50\% agree).
When using SC, the reliance on PRMs can be eliminated altogether.
Experiments in \cref{sec:mainresults} confirm that our algorithm remains effective when using SC as the aggregation rule.

\paragraph{\cref{alg:methods} with different exploration rules.}
While the base version of \cref{alg:methods} (\elimtext) explores all active queries uniformly at each round, our framework supports more targeted exploration strategies inspired by the pure exploration bandit literature.
For example, Upper Confidence Bound (\ucbtext) prioritizes queries with high empirical reward plus an uncertainty bonus \citep{kalyanakrishnan2012pac, jamieson2014lil}, while gap-based sampling (\gaptext) focuses on queries near the elimination threshold $\gamma$, allocating compute inversely proportional to the estimated reward gap \citep{locatelli2016optimalalgorithmthresholdingbandit}.
We also propose a novel entropy-based rule (\entropytext) that selects queries with more diverse response patterns, as measured by empirical entropy, and encourages exploration of under-sampled queries.
Experiments in \cref{sec:analysis_empirical} show that \entropytext is particularly effective across extremely difficult query sets.
We defer formulations of these strategies to \cref{app:extensions}. For all the exploration rules, to prevent over-allocation of compute on difficult queries, we additionally introduce a per-query cap $\texttt{max\_sample} \in \bbN$, which limits the number of generated responses for any individual query.
\looseness=-1

\paragraph{\cref{alg:methods} with fine-grained token controls.}
The default version of \cref{alg:methods} models compute cost as the number of response generations.
However, it can be easily extended to track and control token-level usage.
At each iteration, the algorithm can record token consumption and stop once the total token budget is reached.
Alternatively, one can impose fine-grained token caps per generation.
We evaluate this variant in \cref{subsec:ablation} and find that \cref{alg:methods} continues to outperform baselines under the same token budget.

\paragraph{\cref{alg:methods} with streaming queries.}
In streaming settings, queries arrive sequentially, i.e., only the current query $x_t$ is accessible at round $t$.
To adapt \cref{alg:methods} to this setting, we modify line~\ref{line:elim1} to focus solely on $x_t$ while keeping the rest of the framework unchanged.
Also, to prevent over-allocation of compute on difficult queries, we additionally introduce a per-query cap $\texttt{max\_sample} \in \bbN$, which limits the number of generated responses for any individual query.
This constraint enforces a local trade-off between exploration and exploitation and promotes balanced compute usage across the query stream.
Ablation results in \cref{subsec:ablation} show that this streaming variant remains competitive with our original method.

\subsection{Theoretical insights on compute efficiency}
\label{sec:theoretical_analysis}

A key strength of \cref{alg:methods} (and its variants in \cref{sec:extensions} with different exploration rules) lies in their ability to adapt compute based on estimated query difficulty: easier queries get fewer samples, while harder ones are given more when needed.

To better understand the advantage of \cref{alg:methods} over uniform compute allocation, we consider the following probabilistic model. 
For each query $x \in \cS$, we model the correctness of the LLM's response in a single, independent generation as a Bernoulli random variable with parameter $\Delta_x \in (0, 1)$. That is, $X \sim \text{Bernoulli}(\Delta_x)$, where $X =1$ if the LLM answers the query correctly and $X=0$ otherwise.

To ensure the reward oracle is compatible with this probability model and the specified threshold $\gamma$, we make the following assumption.

\begin{assumption}
    \label{assumption:reward_model}
    For any query $x \in \cS$ and any randomly generated response $y$. $y$ correctly answers $x$ if and only if the reward oracle $r$ assigns a score $r(x, y) \geq \gamma$.
\end{assumption}

\cref{assumption:reward_model} ensures that the elimination decision is aligned with the reward oracle and the threshold, allowing us to focus on the analysis of the adaptive design in \cref{alg:methods}.
This assumption is satisfied by the ground truth reward oracle, and holds approximately when the reward model is sufficiently accurate---a condition empirically supported by recent advances in high-quality process reward models \citep{zhang2025lessons}.
\looseness=-1

Suppose $K = O(1)$,
we derive the following quantitative comparison between \cref{alg:methods} and uniform compute allocation.

\begin{restatable}{theorem}{thmSampleComplexity}
  \label{thm:sample_complexity}
  Assume \cref{assumption:reward_model}, and fix any $\delta \in (0,1)$. 
  To output correct responses for all queries in $\cS$ with probability at least $1-\delta$, \cref{alg:methods} requires a total budget  
  $B_\oursmath = \wt \Theta(\sum_{x \in \cS} \frac{1}{\Delta_x})$. This matches the information-theoretic lower bound up to logarithmic factors.
 In contrast, a uniform allocation strategy requires budget $B_{\unifmath} = \wt \Omega (\max_{x \in \cS} \frac{\abs{\cS}}{\Delta_x})$ to achieve the same guarantee.
\end{restatable}

\cref{thm:sample_complexity} highlights the efficiency advantage of \cref{alg:methods} over uniform allocation. In particular, \cref{alg:methods} requires a total budget of order $\wt \Theta(\sum_{x \in \cS} \frac{1}{\Delta_x})$, whereas uniform allocation requires $\wt \Omega (\max_{x \in \cS} \frac{\abs{\cS}}{\Delta_x})$. Thus, by adapting to query difficulty, \cref{alg:methods} can significantly reduce total compute. For example, if $\abs{\cS}=n$ and $\Delta_{x_i}=i/n$, then $B_\oursmath= \wt O(n)$ while $B_{\unifmath} = \wt \Omega (n^2)$, so uniform allocation can be a factor of $n$ less efficient.

One may further consider the classical two-stage explore-then-commit (ETC) strategy,
which allocates a fixed exploration budget to each query before committing additional compute based on the observed outcomes \citep{lattimore2020bandit}.
Despite its partial adaptivity, ETC remains fundamentally limited because its first stage performs uniform allocation.
We next show that this structural constraint already leads to strictly larger budget complexity in heterogeneous regimes.

\begin{restatable}{proposition}{propETC}
\label{prop:ETC}
Assume \cref{assumption:reward_model}, and fix any $\delta \in (0,1)$.
Consider any two-stage explore-then-commit algorithm that allocates
a fixed exploration budget $m$ to every query in the first stage.
To output correct responses for all queries in $\cS$ with probability at least $1-\delta$,
any such algorithm must incur total compute $B_\etcmath = \wt \Omega(\sum_{x \in \cS} \max(m, \frac{1}{\Delta_x}))$.
\end{restatable}

In particular, when $m$ is nontrivial and many queries are easy (i.e., when $\frac{1}{\Delta_x}$ is small),
ETC incurs an unavoidable uniform-exploration overhead of $\Theta(m)$ per easy query,
whereas \cref{alg:methods} only spends $\wt O(1/\Delta_x)$.
This establishes a strict separation between ETC and our fully adaptive allocation.
Consistent with this theory, our empirical results in \cref{sec:experiments} also show that the ETC baseline (denoted as \textsc{Two Stage}) is substantially outperformed by our algorithm.

\section{Experiments}
\label{sec:experiments}
We describe experimental setups in \cref{sec:setup}, present main results in \cref{sec:mainresults}, offer further analysis in \cref{sec:analysis_empirical}, and report ablations in \cref{subsec:ablation}.
Additional experimental details and results are deferred to \cref{app:experiments}.

\subsection{Experimental setup}
\label{sec:setup}
\paragraph{Datasets.}  
We examine the performance of our algorithms on standard math and code benchmarks: {MATH-500} and {AIME25} \citep{lightman2023letsverifystepstep, hendrycksmath2021, aime2025} and {LiveCodeBench} \citep{jain2024livecodebenchholisticcontaminationfree}. 
MATH-500 contains 500 math questions, AIME25 contains 30 difficult math questions, and the LiveCodeBench contains 479 code execution questions that were collected from 5/1/2023 to 12/1/2023. 
From MATH-500, we further construct one challenging subset: {MATH-500-Hard-16}, which contain questions that cannot be correctly answered after allocating 16 units of compute.
Intuitively, this subset consists of the most difficult queries in the MATH-500 dataset.
\looseness=-1

\paragraph{Baselines.} 
We compare our algorithms with the uniform Best-of-$N$ baseline \citep{brown2024largelanguagemonkeysscaling}, referred to as \textsc{Uniform}, and a two-stage baseline, referred to as \textsc{Two Stage} \citep{damani2024learning, wang2025makepennycountdifficultyadaptive}. The \textsc{Two Stage} baseline first uniformly allocates compute to estimate problem difficulty (stage 1) and then allocates the remaining compute proportionally (stage 2); in other words, it operates in an explore-then-commit style. 
For the \textsc{Two Stage} algorithm, we vary the stage 1 compute ratio from $\crl{25\%, 50\%, 75\%}$ and report the best results.
We report the performance of our \cref{alg:methods} (\elimtext) and its variants introduced in \cref{sec:extensions}.

\paragraph{Models and metrics.}
We conduct experiments with commonly used LLMs of various sizes, including \texttt{Llama-3.2-1B-Instruct} and \texttt{Llama-3.1-8B-Instruct} \citep{grattafiori2024llama}, as well as more recently developed reasoning models \texttt{DeepSeek-R1-Distill-Llama-8B} \citep{deepseekai2025deepseekr1incentivizingreasoningcapability}, \texttt{Qwen3-4B} \citep{yang2025qwen3technicalreport}, and  \texttt{Gemini-2.5-Flash-Lite} \citep{comanici2025gemini}.
For MATH-500, we use different reward oracles depending on the base model family: for Llama models, we use the PRM \texttt{Qwen2.5-Math-PRM-7B} \citep{zhang2025lessons}; for \texttt{Gemini-2.5-Flash-Lite}, we use an LLM-as-a-judge oracle based on \texttt{Gemini-3.0-Flash}. In addition, we also report results under the ground-truth (GT) oracle when available.\footnote{When using the GT reward oracle, accuracy and coverage are equivalent. In this case, we only report coverage.}
For AIME25, we use the Self-Consistency (SC) variants of baselines and our methods, as recent work shows that SC is more effective for reasoning models \citep{wang2025thinkdeepthinkfast}.
For LiveCodeBench, since correctness can be deterministically verified by code execution, we use the GT reward oracle and report only the coverage metric (equivalent to accuracy).
We conduct experiments under average compute budgets of $\{4, 8, 16, 32\}$ and report results averaged over $4$ random runs, with shaded regions in plots representing $\pm 0.5$ standard deviations.

\subsection{Main results}\label{sec:mainresults}

\begin{figure}[t]
  \centering
  \begin{subfigure}[t]{.24\linewidth}
    \centering
    \includegraphics[width=\linewidth]{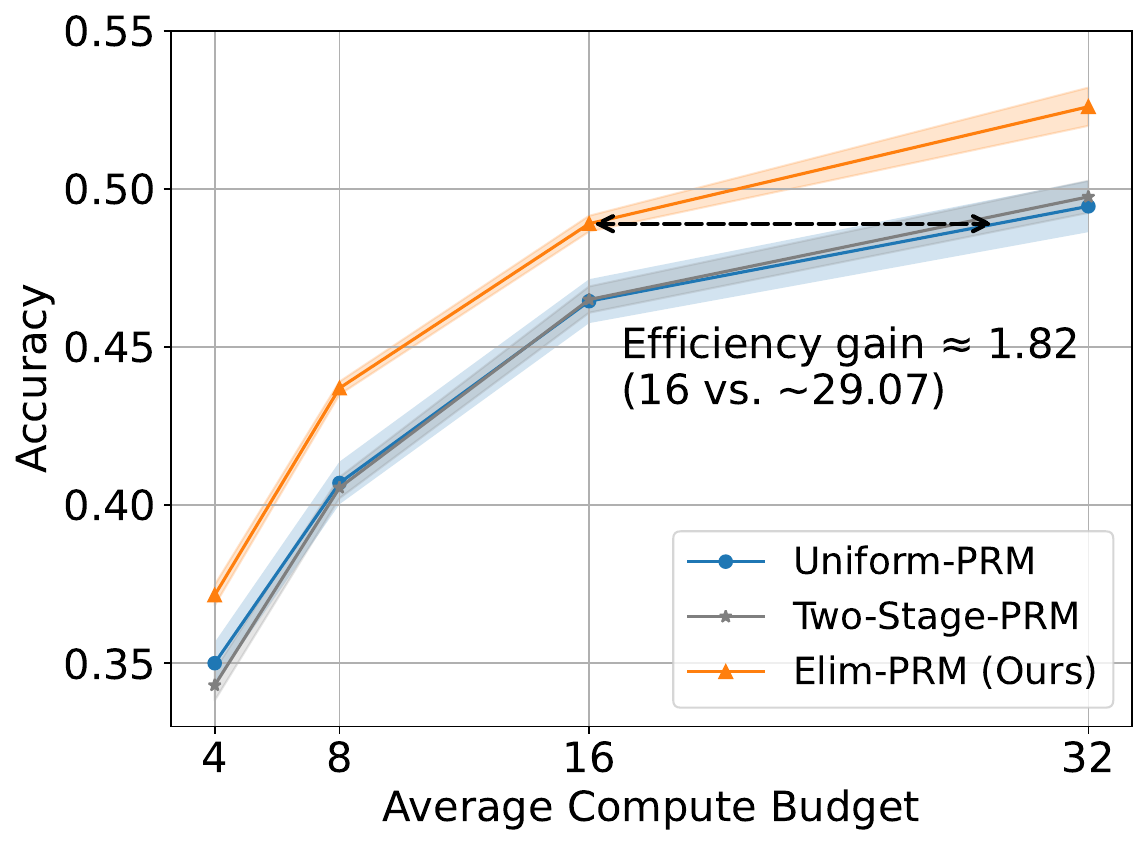}
    \caption{\texttt{Llama-3.2-1B}}
    \label{fig:acc32}
  \end{subfigure}\hfill
    \begin{subfigure}[t]{.24\linewidth}
    \centering
    \includegraphics[width=\linewidth]{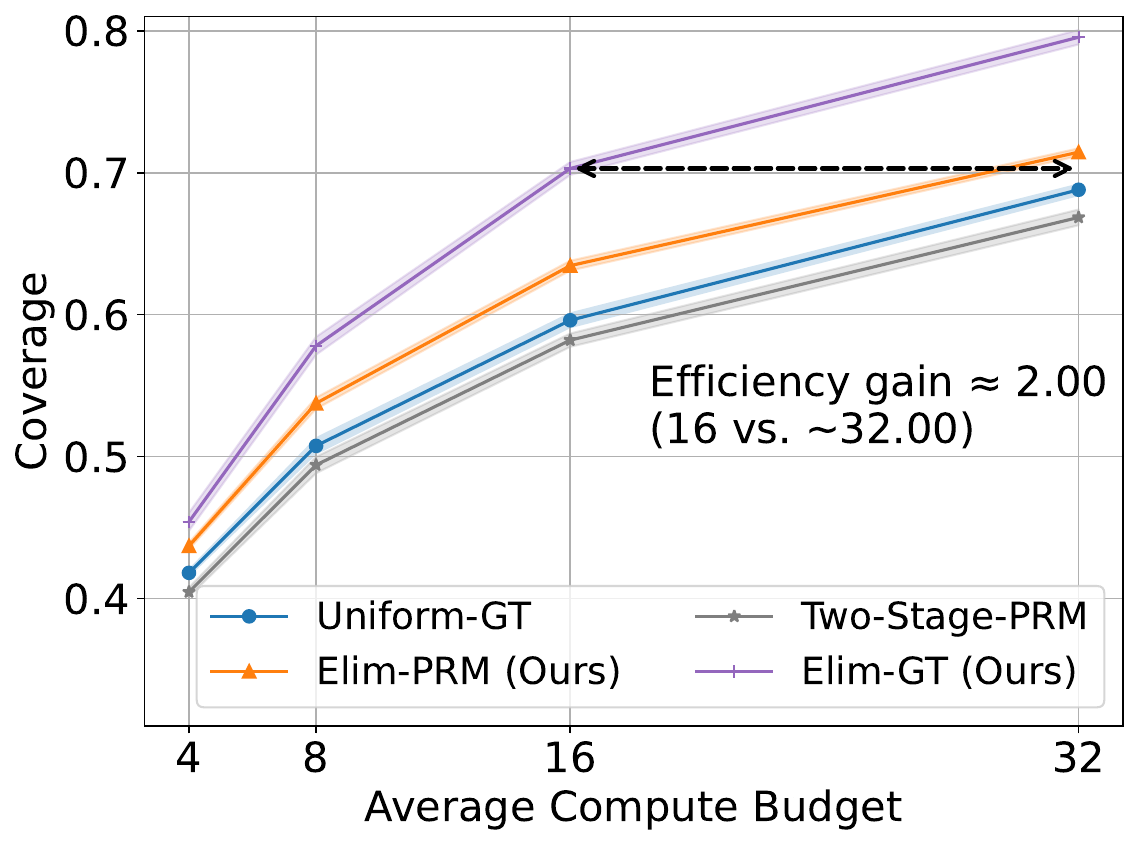}
    \caption{ \texttt{Llama-3.2-1B}}
    \label{fig:cov32}
  \end{subfigure}\hfill
  \begin{subfigure}[t]{.24\linewidth}
    \centering
    \includegraphics[width=\linewidth]{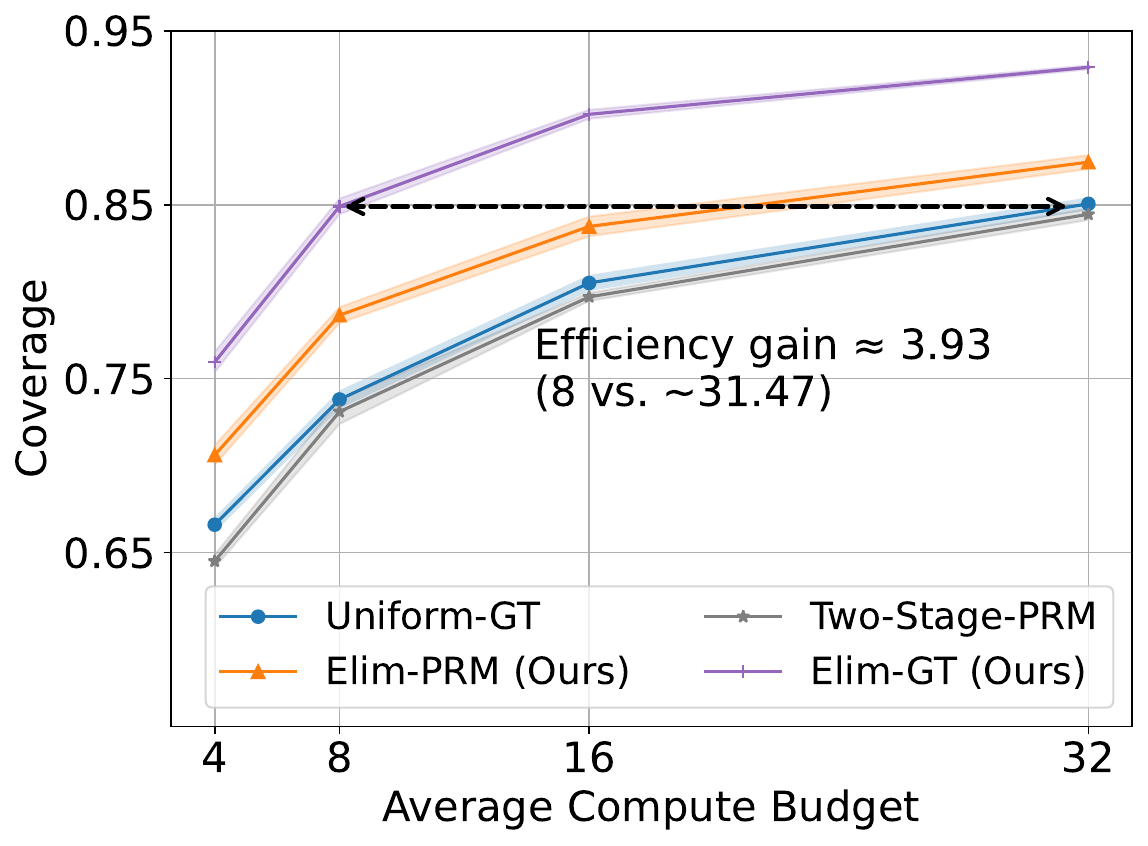}
    \caption{ \texttt{Llama-3.1-8B}}
    \label{fig:acc31}
  \end{subfigure}
   \begin{subfigure}[t]{.24\linewidth}
    \centering
    \includegraphics[width=\linewidth]{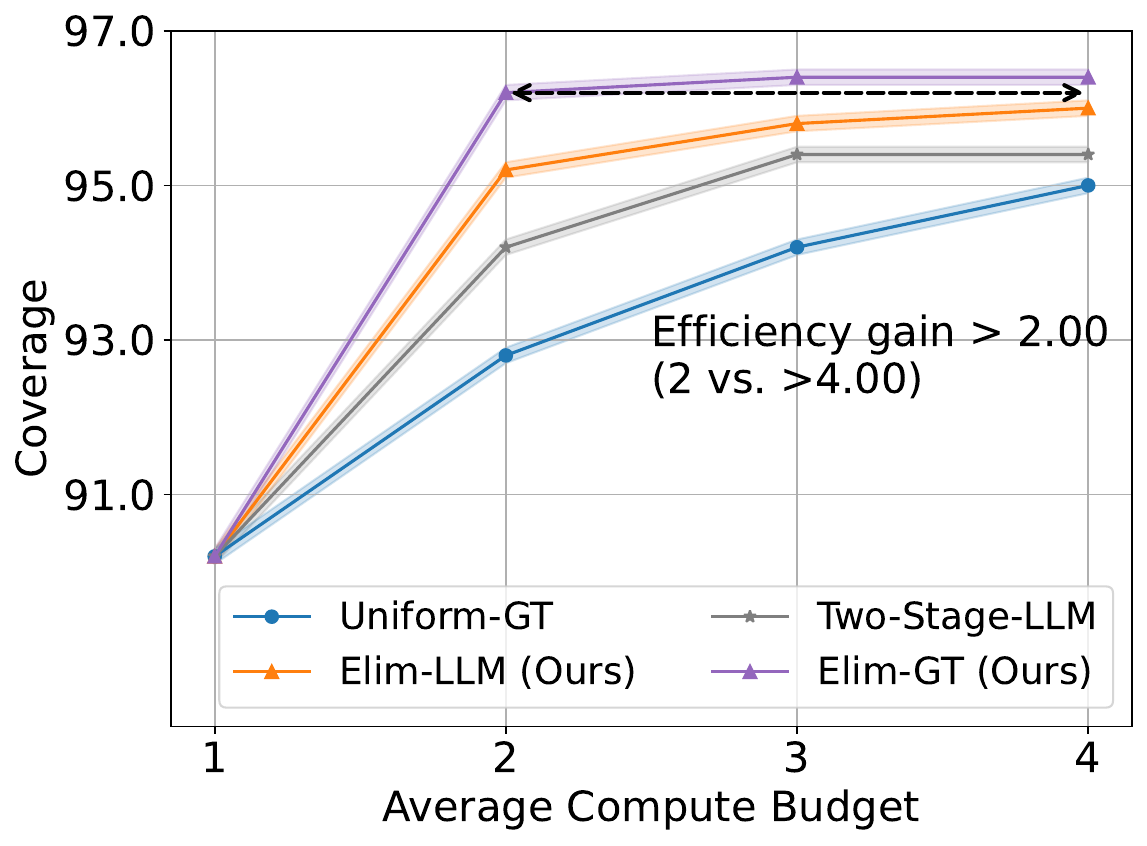}
    \caption{ \texttt{Gemini-2.5-flash-lite}}
    \label{fig:cov_gemini}
  \end{subfigure}

  \caption{Accuracy and coverage comparisons on \textsc{MATH-500} across models of different sizes. Accuracy results of \texttt{Llama-3.1-8B-Instruct} and  \texttt{Gemini-2.5-flash-lite} are presented in \cref{fig:illustration}.}
  \label{fig:acc_cov_grid}
\end{figure}

\textbf{MATH-500 results.} 
\cref{fig:acc_cov_grid} presents experimental results on the MATH-500 dataset across two LLMs and two evaluation metrics (the accuracy result of \texttt{Llama-3.1-8B-Instruct} on MATH-500 is presented in the first plot of \cref{fig:illustration}).
Across all configurations, all variants of our \cref{alg:methods} consistently outperform both baselines.
Under the accuracy metric, when the average compute budget is 16, our method achieves a 2.50\% absolute improvement (7.37\% relative) on \texttt{Llama-3.2-1B-Instruct}; this corresponds to a $1.82\times$ efficiency gain as shown on plot (a): \uniftext takes $1.82\times$ compute to achieve the same performance.  
For \texttt{Llama-3.1-8B-Instruct}, we observe a 1.40\% absolute improvement (4.11\% relative), with a $2\times$ efficiency gain (\cref{fig:illustration}, first plot).
For the coverage metric, when the average compute budget is 16, our method yields a 10.70\% absolute improvement (17.95\% relative) on \texttt{Llama-3.2-1B-Instruct}, resulting in a $2\times$ efficiency gain (\cref{fig:acc_cov_grid}, plot (b)).\footnote{Under the GT reward oracle, we report only the performance of \textsc{Elimination}, as other variants yield similar results. See \cref{app:additional_experiments} for full comparisons.}  
When the average budget is 8, we observe an 11.10\% absolute gain (15.04\% relative) on \texttt{Llama-3.1-8B-Instruct}, yielding a $3.93\times$ efficiency gain (\cref{fig:acc_cov_grid}, plot (c)). 

We further evaluate our method on the frontier model \texttt{Gemini-2.5-Flash-Lite} with reasoning enabled. Given its strong base performance, we use smaller average compute budgets of $\{1,2,3,4\}$.
Because existing PRMs are not reliable at this performance level, we instead use \texttt{Gemini-3.0-Flash} as an LLM-as-a-judge oracle for \texttt{Gemini-2.5-Flash-Lite}.
Under the accuracy metric, our method improves performance by $3.4\%$ absolute ($3.66\%$ relative), corresponding to a $>2.00\times$ efficiency gain (\cref{fig:illustration}, second plot). Under the coverage metric, we obtain a $2.4\%$ absolute improvement ($2.64\%$ relative), again exceeding a $2\times$ efficiency gain (\cref{fig:acc_cov_grid}, plot (d)).
\looseness=-1

\textbf{AIME25 results.}
\cref{fig:reasoning_acc_cov_grid} reports results on the AIME25 dataset with \texttt{Qwen3-4B} (accuracy result of \texttt{Qwen3-4B} on AIME25 is presented at the third plot of \cref{fig:illustration}).
Following \citet{yang2025qwen3technicalreport}, we enable the reasoning ability of \texttt{Qwen3-4B} and set the max token length to 38,912.
Under the accuracy metric (\cref{fig:illustration}, third plot), we are able to gain an absolute of 3.33\% performance gain (4.60\% relative) when the average compute budget is 8 on \texttt{Qwen3-4B}, yielding a $4.00\times$ efficiency gain. 
For the coverage metric (\cref{fig:reasoning_acc_cov_grid}, left plot), when the average compute budget is 16, we observe a 10.82\% absolute gain (14.44\% relative) on \texttt{Qwen3-4B}, yielding a $4.00\times$ efficiency gain.

\textbf{LiveCodeBench results.}
\cref{fig:reasoning_acc_cov_grid} (middle and right) presents results on the LiveCodeBench with \texttt{Llama} models of different sizes, and the fourth plot of \cref{fig:illustration} presents results of \texttt{DeepSeek-R1-Distill-Llama-8B}.
As described in \cref{sec:setup}, we use the GT reward oracle and report coverage, which is equivalent to accuracy in this setting.  
We report results for the \elimtext variant only, as \ucbtext and \gaptext behave identically to \elimtext under the GT oracle.
Across all compute budgets, our method consistently outperforms uniform allocation. With an average compute budget of 16, \texttt{Llama-3.2-1B-Instruct} achieves a 6.47\% absolute improvement (11.63\% relative), corresponding to a $1.98\times$ efficiency gain (\cref{fig:reasoning_acc_cov_grid}, middle plot).  
With an average compute budget of 8, \texttt{Llama-3.1-8B-Instruct} achieves a 7.41\% absolute improvement (14.40\% relative), corresponding to a $3.11\times$ efficiency gain (\cref{fig:reasoning_acc_cov_grid}, right plot).
With an average compute budget of 8, \texttt{DeepSeek-R1-Distill-Llama-8B} achieves a 9.97\% absolute improvement (12.30\% relative), corresponding to a $4.00\times$ efficiency gain (\cref{fig:illustration}, fourth plot).

\begin{figure}[t]
  \centering
  \begin{subfigure}[t]{.32\linewidth}
    \centering
    \includegraphics[width=\linewidth]{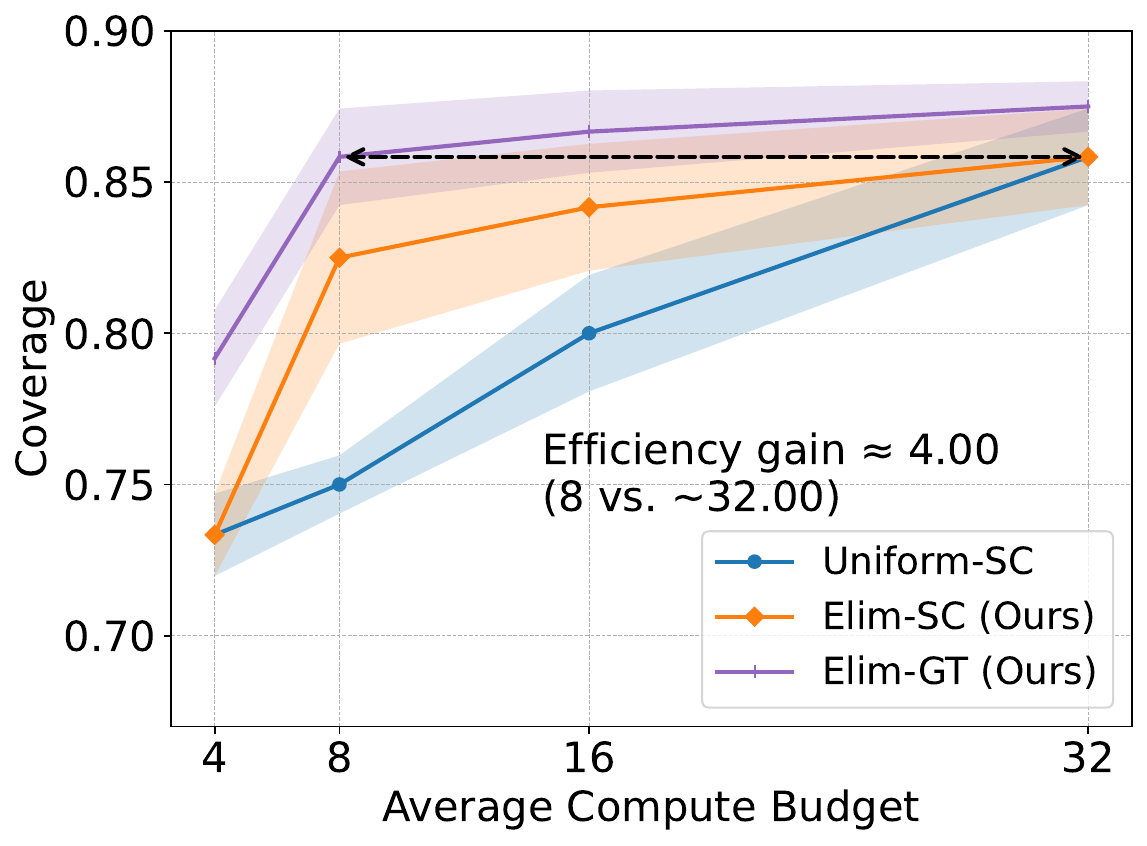}
    \caption{\texttt{Qwen3-4B}}
    \label{fig:acc17}
  \end{subfigure}\hfill
  \begin{subfigure}[t]{.32\linewidth}
    \centering
    \includegraphics[width=\linewidth]{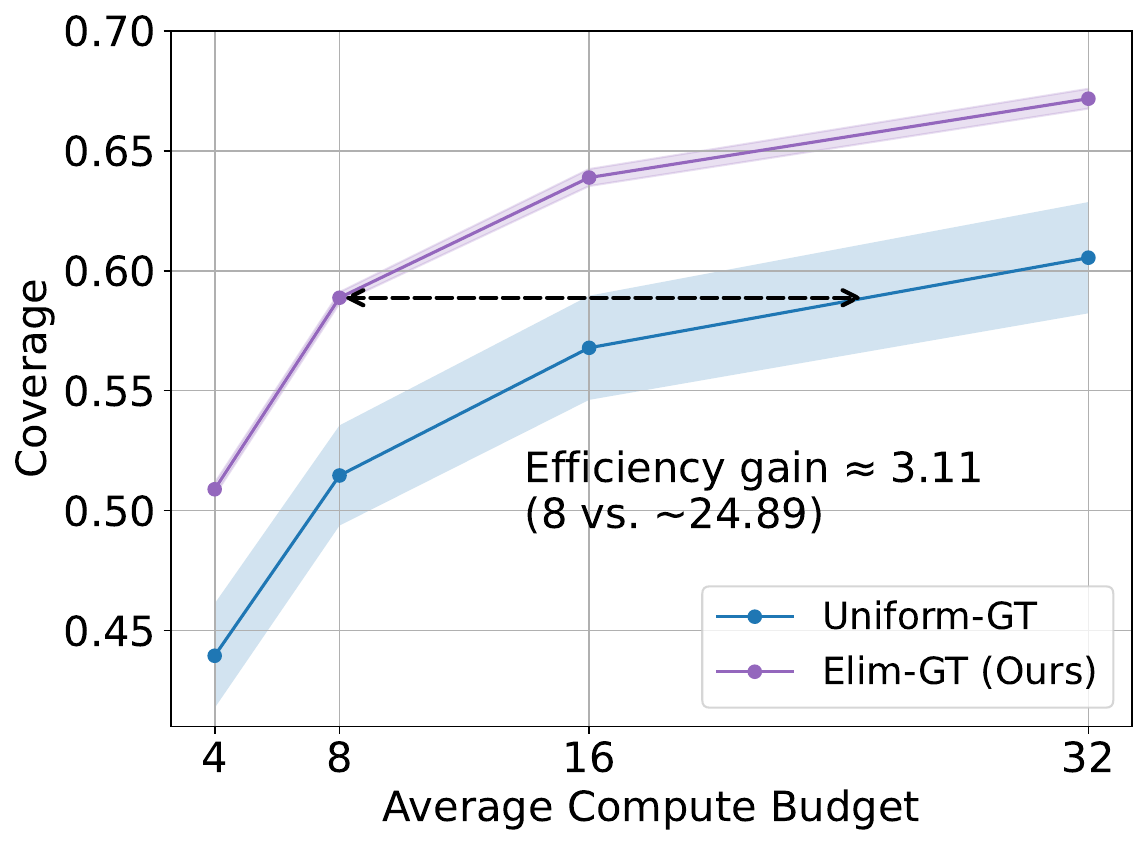}
    \caption{\texttt{Llama-3.2-1B-Instruct}}
    \label{fig:acc4B}
  \end{subfigure}\hfill
  \begin{subfigure}[t]{.32\linewidth}
    \centering
\includegraphics[width=\linewidth]{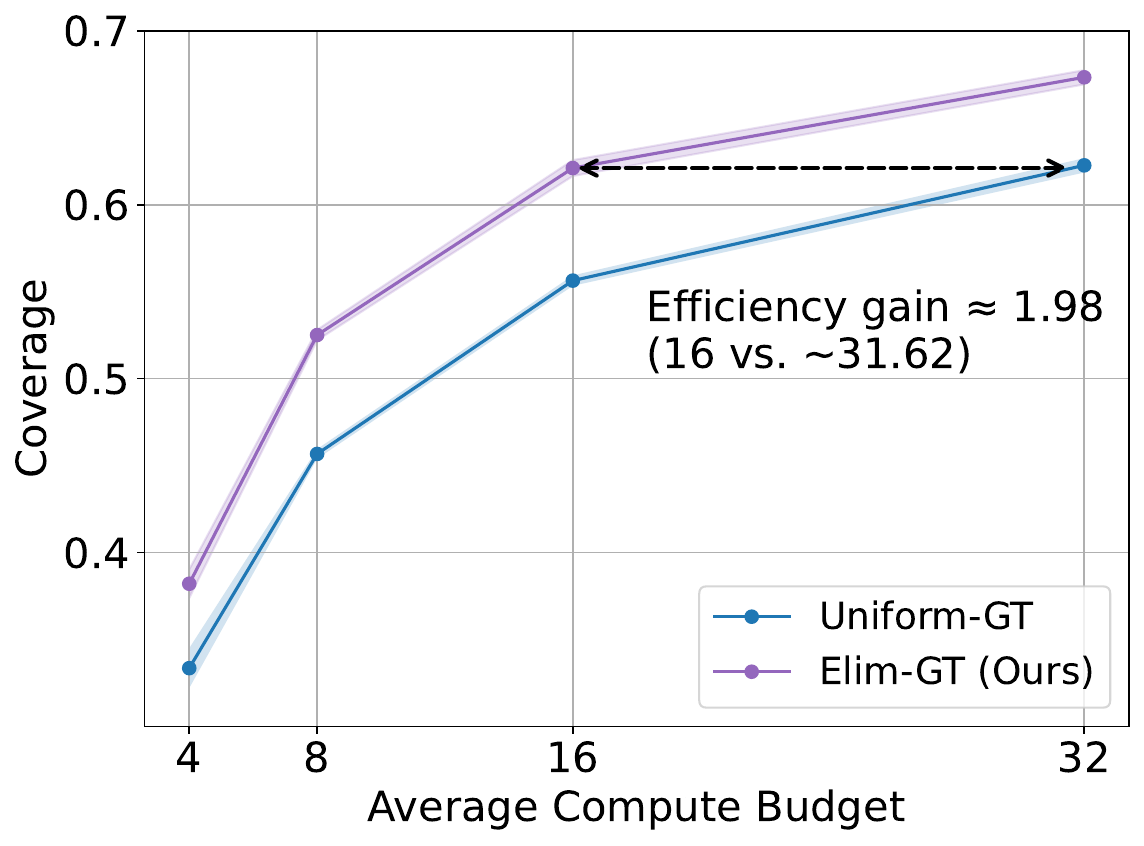}
\caption{\texttt{Llama-3.1-8B-Instruct}}
    \label{fig:cov17}
  \end{subfigure}

  \caption{Results on \textsc{AIME25} (\emph{left}) and LiveCodeBench (\emph{middle and right}). Accuracy result of \texttt{Qwen3-4B} on AIME25 and coverage result of \texttt{DeepSeek-R1-Distill-Llama-8B} on LiveCodeBench are presented in \cref{fig:illustration}.}
  \label{fig:reasoning_acc_cov_grid}
\end{figure}

\textbf{MATH-500-Hard result.}
\cref{fig:ucb-combined} (left) presents experimental results on the MATH-500-Hard-16 dataset, which was constructed to include the most challenging questions in the MATH-500 benchmark. 
We evaluate performance using the GT reward oracle, as PRM-based scores are less reliable on these difficult questions. 
On this dataset, baseline methods and the vanilla version of our \cref{alg:methods} (\elimtext) doesn't performance well, as many questions can not be solved by the base LLM.
However, we find alternative exploration rules introduced in \cref{sec:extensions}---particularly \entropytext{}---achieve significantly better results. 
These findings highlight the benefits of incorporating more nuanced exploration strategies, such as those developed in \cref{sec:extensions}, for effective compute allocation on challenging benchmarks. We defer more experiment results to \cref{app:math_hard}.

\begin{figure}[t]
  \centering
  \begin{subfigure}[t]{.32\linewidth}
    \centering
    \includegraphics[width=\linewidth]{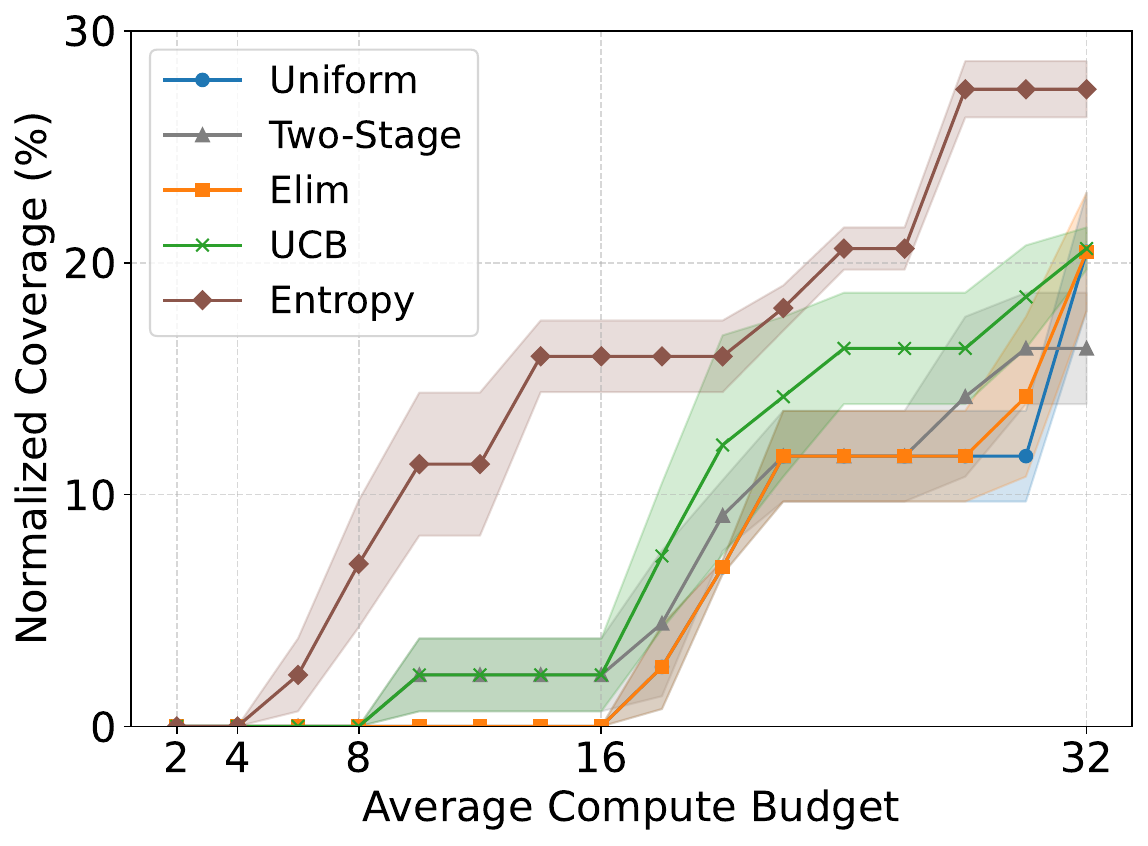}
    \label{fig:mathhard}
  \end{subfigure}\hfill
  \begin{subfigure}[t]{.32\linewidth}
    \centering
    \includegraphics[width=\linewidth]{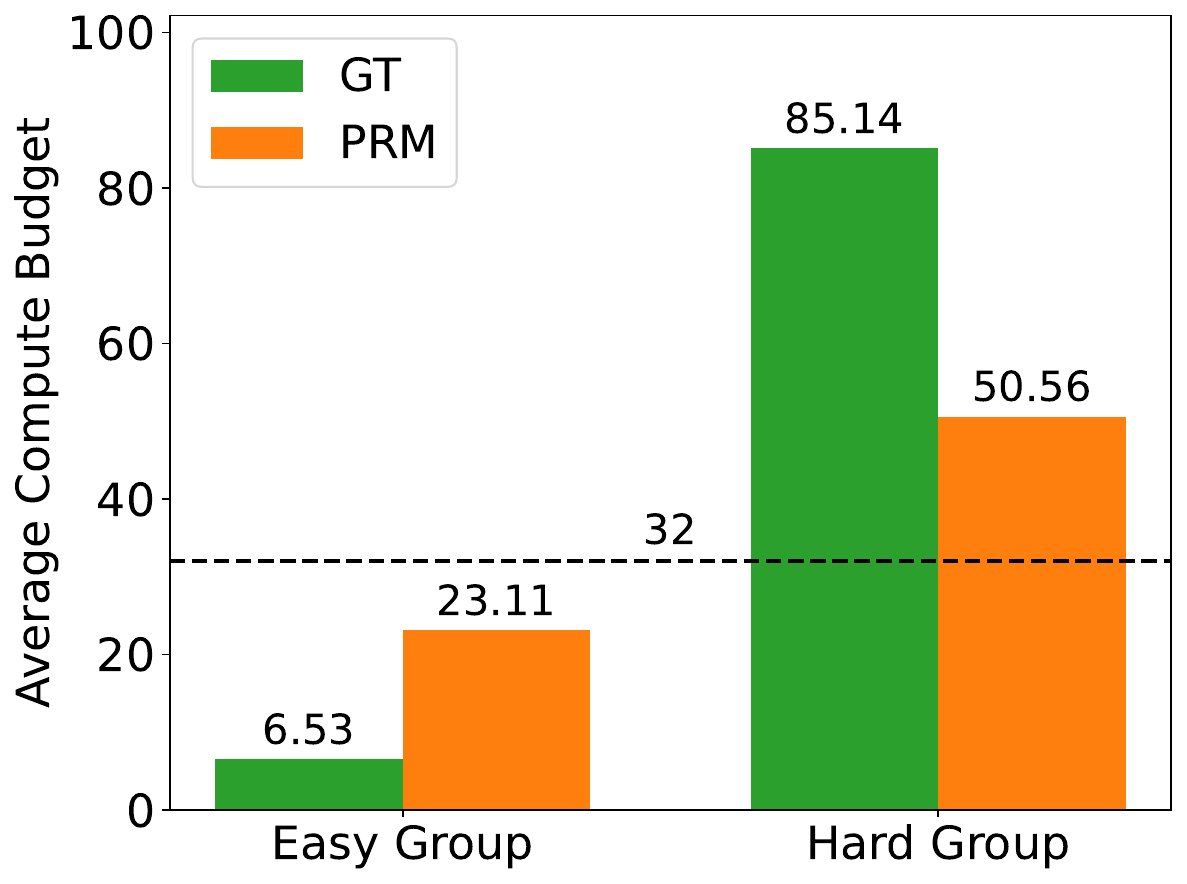}
    \label{fig:methods}
  \end{subfigure}\hfill
  \begin{subfigure}[t]{.32\linewidth}
    \centering
    \includegraphics[width=\linewidth]{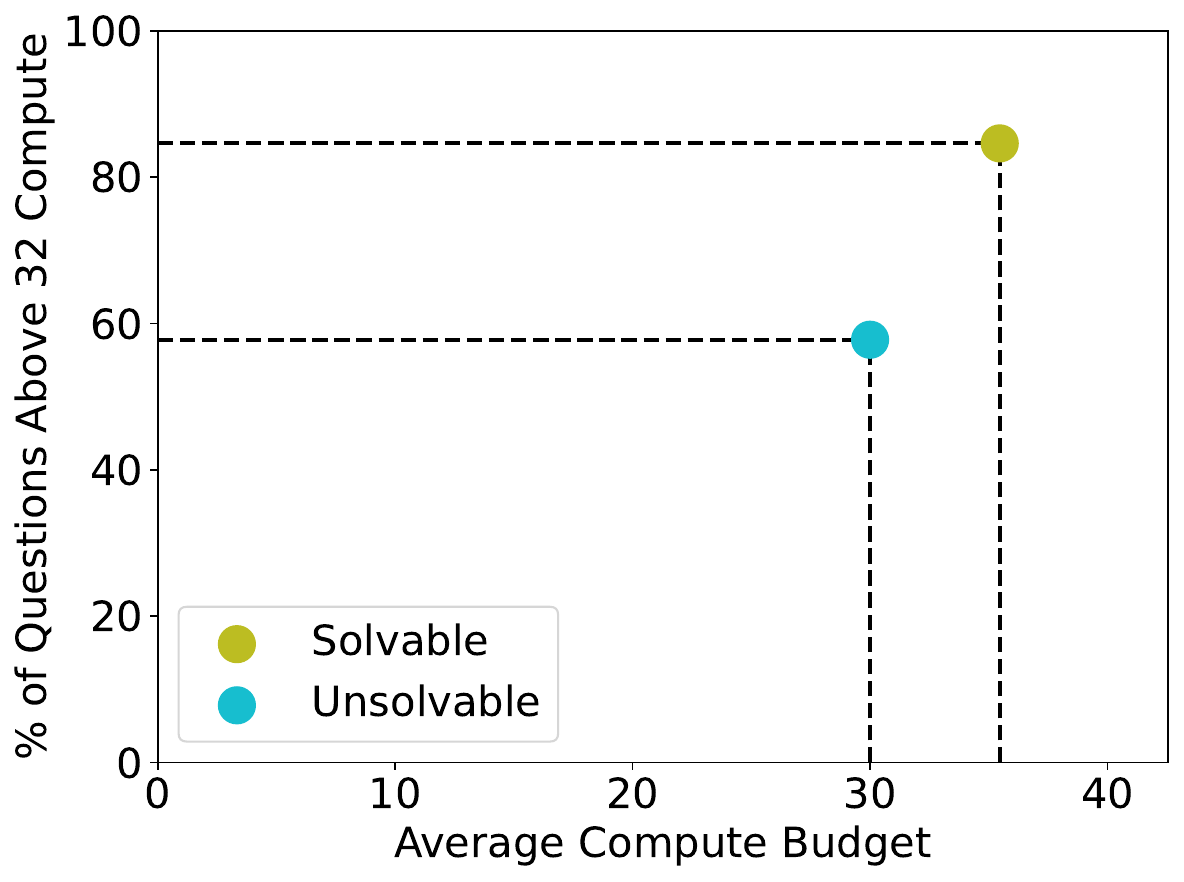}
    \label{fig:alloc-entropy}
  \end{subfigure}

  \caption{Results on MATH-500-Hard with \texttt{Llama-3.2-1B-Instruct} and an average compute budget of $32$. 
  \emph{Left:} Coverage result on \textsc{MATH-500-Hard-16}. 
  \emph{Center:} Allocation behavior of \cref{alg:methods} for easy vs.\ hard groups. 
  \emph{Right:} Allocation behavior of \entropytext{} for solvable vs.\ unsolvable groups.}
  \label{fig:ucb-combined}
  \vspace{-5 pt}
\end{figure}

\subsection{Analysis on the advantages of strategic compute allocation}
\label{sec:analysis_empirical}

We conduct further empirical analyses to illustrate the benefits of strategic compute allocation in two settings:
(1) on standard datasets containing both easy and hard queries, and  
(2) on challenging datasets containing both solvable and unsolvable queries.
All experiments are conducted using \texttt{Llama-3.2-1B-Instruct} with an average compute budget of $32$.

\textbf{Strategic allocation on standard datasets.}
In the first analysis, we partition the MATH-500 dataset into two subsets:  
queries that can be correctly answered with at most $32$ units of compute (easy group), and those that cannot (hard group).  
Intuitively, the easy group consists of questions that require less than 32 units of compute to solve, while the hard group includes questions that would benefit from additional compute.  
  In the middle plot of \cref{fig:ucb-combined}, we visualize the compute allocation of \cref{alg:methods} under both PRM and GT reward oracles.  
Compared to uniform allocation, our algorithm allocates fewer resources to easy queries and more to hard ones.  
This demonstrates the ability of \cref{alg:methods} to strategically allocate compute---reserving effort for harder queries that need it most.

\textbf{Strategic allocation on challenging datasets.}
In the second analysis, we consider the MATH-500-Hard-16 dataset and divide it into solvable queries and unsolvable ones, where the latter cannot be correctly answered even after allocating 500 units of compute.  
In such settings, effective allocation should prioritize the solvable subset, as investing in unsolvable queries leads to wasted compute.  
The right plot of \cref{fig:ucb-combined} shows that under a $32$-unit compute budget, \entropytext allocates more compute on average to solvable queries, and a larger fraction of them receive more than 32 samples.  
This demonstrates that our method learns to concentrate compute on tractable instances, avoiding waste on queries unlikely to be resolved.

To understand why \entropytext behaves this way, we inspect model outputs on these challenging questions in detail. We observe that unsolvable queries often yield invalid responses (e.g., incomplete or poorly formatted), leading to lower entropy across generations.  
In contrast, solvable queries tend to produce more diverse and well-formed outputs, resulting in higher entropy. We defer detailed experiments and analysis of \entropytext extension to \cref{app:math_hard}.  

\subsection{Ablation studies} 
\label{subsec:ablation}

\begin{figure}[t]
  \centering
  \begin{subfigure}[t]{.32\linewidth}
    \centering
    \includegraphics[width=\linewidth]{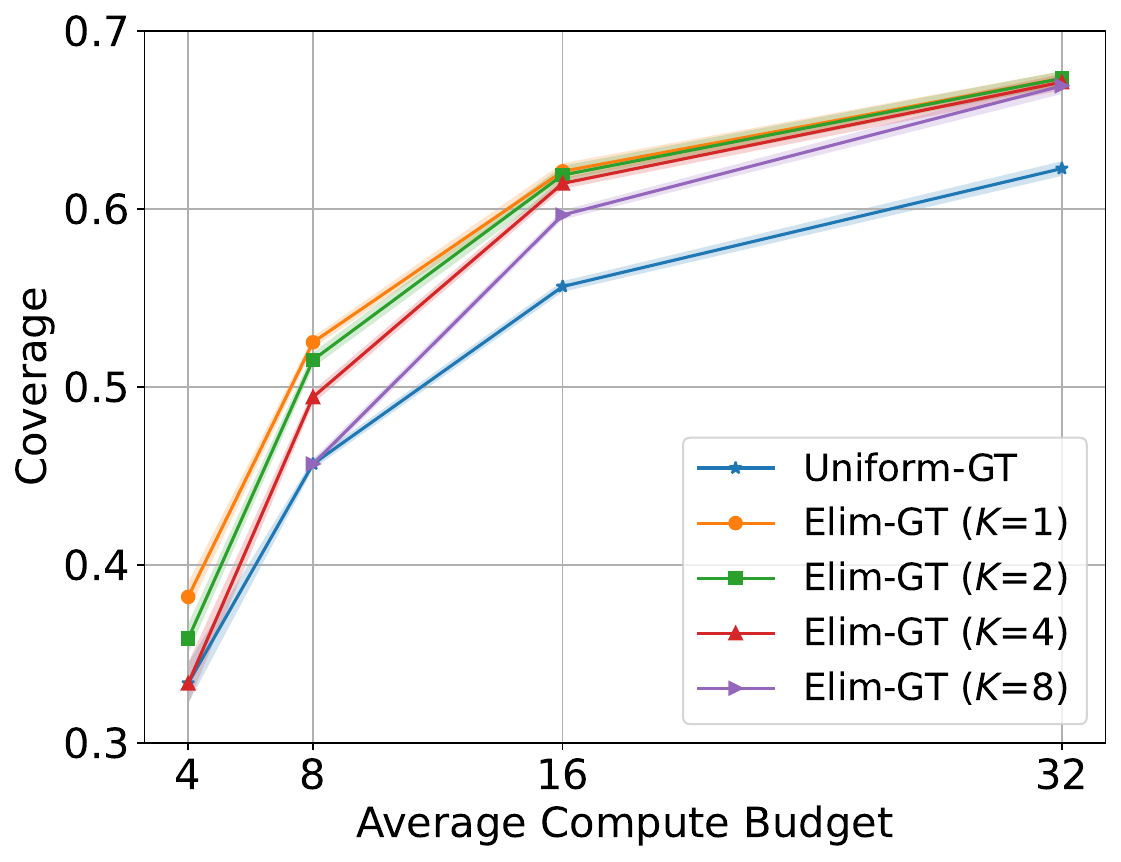}
  \end{subfigure}\hfill
  \begin{subfigure}[t]{.32\linewidth}
    \centering
    \includegraphics[width=\linewidth]{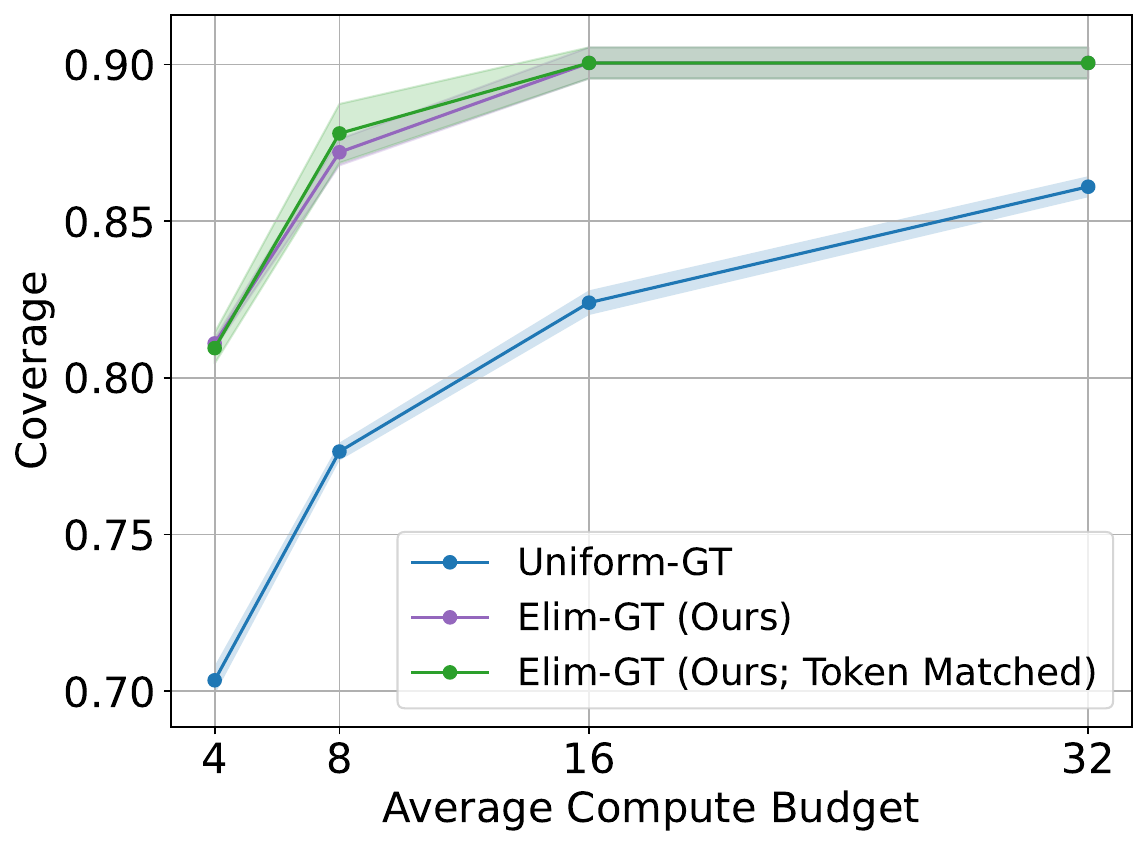}
  \end{subfigure}\hfill
  \begin{subfigure}[t]{.32\linewidth}
    \centering
    \includegraphics[width=\linewidth]{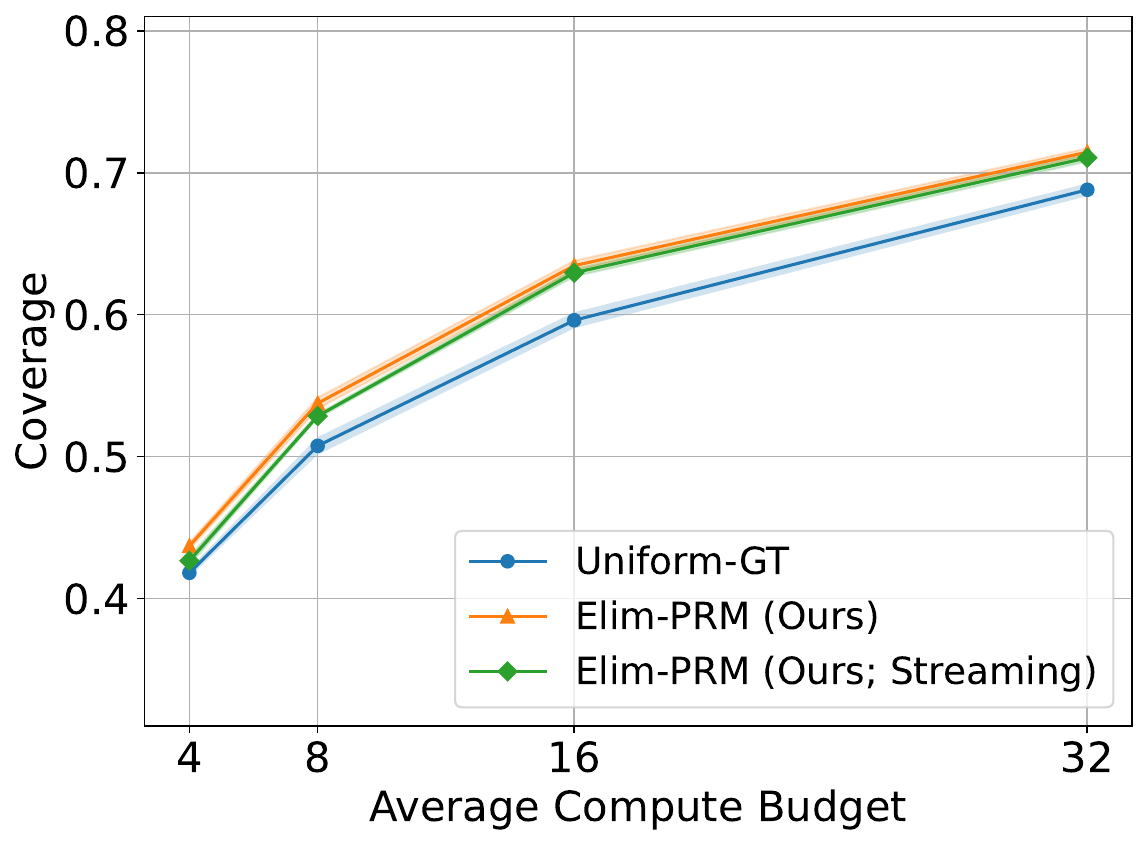}
  \end{subfigure}

  \caption{Ablation studies for \cref{alg:methods}. 
  \emph{Left:} Effect of $K$ on LiveCodeBench with \texttt{Llama-3.2-1B-Instruct}. 
  \emph{Center:} Performance of the token-controlled variant on LiveCodeBench with \texttt{DeepSeek-R1-Distill-Llama-8B}. 
  \emph{Right:} Performance of the streaming variant on MATH-500 with \texttt{Llama-3.2-1B-Instruct}.}
  \label{fig:ablation-two}
  \vspace{-5 pt}
\end{figure}

\textbf{\cref{alg:methods} with different $K$.}
All main experiments in \cref{sec:mainresults} use the default setting $K = 1$. 
Smaller values of $K$ enable finer-grained adaptive allocation and are generally preferred for maximizing performance.
In \cref{fig:ablation-two} (left), we conduct an ablation study with $K \in \{1, 2, 4, 8\}$ on the LiveCodeBench dataset using \texttt{Llama-3.2-1B-Instruct}.
Across all values of $K$, our method consistently outperforms the uniform baseline.
While larger $K$ reduces allocation granularity—making performance closer to uniform allocation under tight budgets—the gap narrows as the average compute budget increases.
These results show that \cref{alg:methods} is robust to the choice of $K$.

\textbf{\cref{alg:methods} with token controls.}
We evaluate \cref{alg:methods} in a token-controlled setting, where compute is measured by total token usage rather than the number of generations.
To ensure comparability, we match the average token budget used by uniform allocation and discard excess samples when needed.
As shown in \cref{fig:ablation-two} (middle), on the LiveCodeBench dataset with \texttt{DeepSeek-R1-Distill-Llama-8B}, \cref{alg:methods} still outperforms the uniform baseline even under equivalent token budgets.

\textbf{\cref{alg:methods} with streaming queries.}
We also test \cref{alg:methods} in a streaming setting, where queries arrive sequentially and the full query pool is not available in advance.
Using the variant described in \cref{sec:extensions}, we evaluate performance in this setting and report results in \cref{fig:ablation-two} (right).
Our method performs comparably to its pool-based counterpart and still outperforms \textsc{Uniform} with pool access, demonstrating its effectiveness under streaming constraints.

\vspace{-5 pt}
\section{Conclusion}
\label{sec:discussion}
We introduce a new perspective on LLM test-time scaling by formulating strategic compute allocation as a bandit learning problem.
We develop adaptive algorithms that estimate query difficulty on the fly and allocate compute to maximize the fraction of correctly answered queries under a fixed compute budget.
Our framework is flexible and extends naturally to incorporate alternative aggregation and exploration strategies, as well as to support both streaming and token-constrained settings.
We provide theoretical guarantees that strategic compute allocation improves compute efficiency over uniform allocation, and we empirically demonstrate substantial performance improvements---up to 11.10\% on MATH-500, 10.82\% on AIME25, and 11.23\% on LiveCodeBench.
These findings underscore the potential of bandit-based compute allocation for more effective test-time scaling.

\bibliography{refs}

\newpage
\appendix
\section{\cref{alg:methods} with Different Exploration Rules}
\label{app:extensions}

Our main algorithmic framework (\cref{alg:methods}) is presented with a simple exploration rule that explores all queries within the active set (lines \ref{line:elim1}-\ref{line:elim2}).
In practice, this rule can be flexibly extended to incorporate diverse exploration objectives.
Motivated by developments in the bandit pure exploration literature, we introduce several alternative exploration rules in the following.
We use $g(x)$ to denote the response set to query $x$, and $N(x) \ldef \abs{g(x)}$ to denote the number of generations so far.

\begin{itemize}
[leftmargin=10pt, itemindent=*]

\item \textbf{Upper confidence bound (\ucbtext).}
For any active query $x \in \cA$, let $\wh r(x) \ldef \sum_{y_i \in g(x)} r(x, y_i) / N(x)$ denote the empirical average reward based on previously collected responses. 
Let $\lambda > 0$ be a hyperparameter. At each round, the \ucbtext exploration rule selects the query based on the following criteria:
\begin{align*}
   \argmax_{x \in \cA} \,\, \wh r(x) + \lambda N(x)^{- 1/ 2}.
\end{align*}
This exploration rule follows the principle of optimism in the face of uncertainty \citep{kalyanakrishnan2012pac, jamieson2014lil}, and prioritizes on selecting queries that are more likely to be solved (i.e., those with higher average rewards). The term $\lambda N(x)^{- 1/ 2}$ is used to construct the upper confidence bound of the reward.

     \item \textbf{\gaptext.}
For any active query $x \in \cA$, let $\wh r(x) \ldef \sum_{y_i \in g(x)} r(x, y_i) / N(x)$ denote the empirical average reward based on previously collected responses. 
At each round, the \gaptext exploration rule selects the query based on the following criterion: 
\begin{align*}
   \argmin_{x \in \cA} \,\, \prn{\gamma - \wh r(x)} \cdot N(x)^{- 1/ 2}.
\end{align*}
This exploration rule prioritizes queries whose estimated reward is close to the elimination threshold $\gamma$, with a preference toward less-explored queries.
 The weighting term $N(x)^{-1/2}$ ensures that compute is allocated inversely proportional to the reward gap from the elimination threshold \citep{locatelli2016optimalalgorithmthresholdingbandit}. 

\item \textbf{\entropytext.}  For any active query $x \in \cA$, let $\{v_k\}$ be the set of distinct responses in $g(x)$, and define the empirical probability of observing response \( v_k \) as $p_k(x) \ldef |\{i: y_i=v_k, y_i \in g(x) \}| / N(x)$.
     Let $H(x) = - \sum_k p_k(x) \log p_k(x)$ denote the entropy of the empirical response distribution $p(x)$. 
     Let $\lambda > 0$ be a hyperparameter. At each round, the \entropytext exploration rule selects the query based on the following criterion: 
\begin{align*}
   \argmax_{x \in \cA} \, \, H(x) + \lambda N(x)^{- 1/ 2}.
\end{align*}
This exploration rule, proposed in our work, prioritizes queries that elicit a more diverse set of responses, as indicated by higher entropy.  
The term \( \lambda N(x)^{-1/2} \) encourages exploration of under-explored queries by balancing the trade-off between response diversity and sample count.
\end{itemize}

\section{Supporting Results from \cref{sec:theoretical_analysis}}
\label{app:methods}

\subsection{Proof of \cref{thm:sample_complexity}}
\thmSampleComplexity*
\begin{proof}
  We first prove that $B_\oursmath= \wt \Theta(\sum_{x \in \cS} \frac{1}{\Delta_x})$ under the elimination rule of \cref{alg:methods}. Note that under \cref{assumption:reward_model}, a query is eliminated if and only if it is correctly answered.\footnote{Since the elimination rule works for all variants of \cref{alg:methods} introduced in \cref{app:extensions}, the guarantee also holds for these variants.}

For the upper bound, 
we denote $\wb \delta \ldef \delta / \abs{\cS}$, $n_x \ldef \frac{1}{\Delta_x} \log \frac{1}{\wb \delta}$ and consider the following event:
$$E_x \ldef \crl{\text{query $x$ will be correctly answered within $n_x$ random generations}}.$$ 
We know that $E_x$ happens with probability at least $1 - \wb \delta$ as the probability of $\wb E_x$ is upper bounded by $\wb \delta$:
\begin{align*}
    \prn{1 - \Delta_x}^{n_x} \leq e^{- \Delta_x \cdot n_x} = e^{- \Delta_x \cdot \frac{1}{\Delta_x} \log \frac{1}{\wb \delta}} = \wb \delta.
\end{align*}
A union bound over $x \in \cS$ leads to $\bbP(\cup_{x \in \cS} E_x) \geq 1 - \sum_{x \in \cS} \bbP(\wb E_x) \geq 1 - \delta$. As a result, the with probability of at least $1 - \delta$, \cref{alg:methods} and its variants in \cref{app:extensions} output correct responses for all queries with compute budget $\sum_{x \in \cS} n_x = \sum_{x \in \cS} \frac{1}{\Delta_x} \log \frac{\abs{\cS}}{\delta} = \wt O(\sum_{x \in \cS}  \frac{1}{\Delta_x})$. 

For the lower bound,
we quantify the amount of compute $n_x$ needed to correctly answer any fixed query $x \in \cS$
with probability at least $1-\delta$.
Recall that each independent generation answers $x$ correctly with probability $\Delta_x \in (0,1)$.
Therefore, after $n_x$ independent generations, the probability that \emph{none} of them is correct equals
\begin{align*}
    \bbP(\text{query $x$ is not correctly answered within $n_x$ generations})
    =
    (1-\Delta_x)^{n_x}.
\end{align*}
If an algorithm outputs a correct response for \emph{all} queries in $\cS$ with probability at least $1-\delta$,
then in particular it must output a correct response for this fixed query $x$ with probability at least $1-\delta$,
which requires $(1-\Delta_x)^{n_x} \le \delta$. 
Taking logarithms gives $n_x
    \geq
    \frac{\log(1/\delta)}{-\log(1-\Delta_x)}$.
Using the inequality $-\log(1-u) \le \frac{u}{1-u}$ for all $u \in (0,1)$, we further obtain
\begin{align*}
    n_x
    \geq
    \frac{1-\Delta_x}{\Delta_x}\log\frac{1}{\delta}.
\end{align*}
Summing the above inequality over all $x \in \cS$ yields the compute lower bound $\Omega(\sum_{x \in \cS} \frac{1 - \Delta_x}{\Delta_x} \log \frac{1}{\delta})$.
In particular, if $\Delta_x \le 1-c$ for all $x \in \cS$ and some fixed constant $c \in (0, 1)$,
then the lower bound simplifies to $\wt \Omega({\sum_{x \in \cS}\frac{1}{\Delta_x}})$, which matches our upper bound up to logarithmic factors.

As for uniform compute allocation, it allocates the same amount of compute for each query $x \in \cS$. 
We thus only need to quantify the amount of compute $\wb n$ needed to correctly answer $\wb x \ldef \argmin_{x \in \cS} \Delta_x$, the hardest query within $\cS$, with high probability. The above lower bound analysis shows that one needs $\Omega(\frac{1}{\Delta_{\wb x}} \log \frac{1}{\delta})$ compute to correctly answer the $\wb x$ with probability at least $1-\delta$. Since uniform allocation assigns the same budget to every query,
the total compute must satisfy $B_{\unif} = \Omega ({\frac{\abs{\cS}}{\Delta_{\wb x}}\log\frac{1}{\delta}} )=  \wt \Omega (\max_{x \in \cS}\frac{\abs{\cS}}{ \Delta_x})$.
\end{proof}

\subsection{Proof of \cref{prop:ETC}}
\propETC*

\begin{proof}
Suppose the ETC algorithm allocates $m$ generations to every query $x \in \cS$ in the exploration stage.
Let $n_x$ denote the \emph{total} number of generations allocated to query $x$ across both stages.
Then we have $n_x \ge m$ for all $x \in \cS$.
On the other hand, by the same per-query lower bound argument used in the proof of \cref{thm:sample_complexity},
to correctly answer query $x$ with probability at least $1-\delta$ it is necessary that
$n_x = \Omega(\frac{1}{\Delta_x}\log\frac{1}{\delta})$.
Combining these two necessary conditions yields
$n_x = \Omega(\max(m, \frac{1}{\Delta_x}\log\frac{1}{\delta}))$.
Summing over all $x \in \cS$ gives
$B_\etcmath = \Omega(\sum_{x \in \cS}\max(m, \frac{1}{\Delta_x}\log\frac{1}{\delta}))
= \wt \Omega(\sum_{x \in \cS}\max(m, \frac{1}{\Delta_x}))$.
\end{proof}

\section{Other details and results for experiments}
\label{app:experiments}

\subsection{Additional details on experimental setups}

\subsubsection{Additional hyperparameters}
We conduct all experiments on two NVIDIA RTX 6000 Ada GPUs.
We use vLLM \citep{kwon2023efficient} for LLM response generation, with a temperature of $0.6$. 

The extended exploration rules introduced in \cref{app:extensions} require additional hyperparameters. 
For \ucbtext, we set $\lambda = 1$; for \entropytext, we set $\lambda = 3$.
For \ucbtext, \gaptext, and \entropytext, we additionally set a \texttt{max\_samples} hyperparameter to prevent over-allocation of compute on difficult queries. The values of the \texttt{max\_samples} hyperparameter will be introduced along with the experiment results.

\subsubsection{Additional details on MATH-500-Hard datasets}
As discussed in \cref{sec:setup}, we construct the MATH-500-Hard dataset by removing queries from MATH-500 that can be solved with 16 units of compute. On top of this, we construct another dataset by removing queries that can be solved with 8 units of compute, and we call this dataset MATH-500-Hard-8. After removing these relatively easy queries, MATH-500-Hard-8 contains 71 challenging queries, and MATH-500-Hard-16 contains 56 challenging queries. 
We further divide MATH-500-Hard queries into two subsets: the subset that cannot be solved after allocating $M$ compute units (Unsolvable) and the rest (Solvable). 
We set $M = 500$ for \texttt{Llama-3.2-1B-Instruct}, and $M = 350$ for \texttt{Llama-3.1-8B-Instruct}. For normalized coverage, we calculate it based on the percentage of solved questions over all solvable questions.
On these MATH-500-Hard datasets, 
we select the best from $\texttt{max\_samples} \in \crl{36, 48, 64}$
for \textsc{UCB}, \textsc{Gap}, and \entropytext. The setting of \texttt{max\_samples} is discussed in the same way as \cref{sec:extensions}.

\subsubsection{Prompts Selection}
\paragraph{Prompts for MATH-500.}
We include four in-context examples in the prompt for MATH-500 for our Llama models.  
An illustrative example used in our experiments is shown below.

\begin{tcolorbox}[
    colback=blue!5!white,      %
    colframe=blue!75!black,    %
    width=\linewidth,
    arc=2mm,                   %
    boxrule=1pt,               %
    left=3mm, right=3mm,       %
    top=2mm, bottom=2mm,
    title=\textbf{MATH prompts}
]
\noindent\textbf{Problem:} If $\det \mathbf{A} = 2$ and $\det \mathbf{B} = 12,$ then find $\det(\mathbf{A}\,\mathbf{B}).$

\medskip
\noindent\textbf{Solution:} 
We have that \[
\det(\mathbf{A}\,\mathbf{B})
= \det(\mathbf{A})\,\det(\mathbf{B})
= (2)\,(12)
= \boxed{24}.
\]
Final Answer: The final answer is $24$. I hope it is correct.

\medskip
\begin{center}
$\vdots$
\end{center}
\medskip

\noindent\textbf{Problem: \{actual test question\}}

\medskip
\noindent\textbf{Solution:} 
\end{tcolorbox}

\paragraph{Prompts for LiveCodeBench.}
We use the prompt provided on the official GitHub of LiveCodeBench \citep{jain2024livecodebenchholisticcontaminationfree}.
The prompt used in our experiments is shown below.

\begin{tcolorbox}[
    colback=blue!5!white,      %
    colframe=blue!75!black,    %
    width=\linewidth,
    arc=2mm,                   %
    boxrule=1pt,               %
    left=3mm, right=3mm,       %
    top=2mm, bottom=2mm,
    title=\textbf{LiveCodeBench CoT prompt}
]
You are given a Python function and an assertion containing an input to the function. Complete the assertion with a literal (no unsimplified expressions, no function calls) containing the output when executing the provided code on the given input, even if the function is incorrect or incomplete. Do NOT output any extra information. Execute the program step by step before arriving at an answer, and provide the full assertion with the correct output in [ANSWER] and [/ANSWER] tags, following the examples.

[PYTHON]
def performOperation(s):
    s = s + s
    return "b" + s + "a"
assert performOperation(s = "hi") == ??
[/PYTHON]

[THOUGHT]
Let’s execute the code step by step:
1. The function \texttt{performOperation} is defined, which takes a single argument \texttt{s}.
2. The function is called with the argument \texttt{"hi"}, so within the function, \texttt{s} is initially \texttt{"hi"}.
3. Inside the function, \texttt{s} is concatenated with itself, so \texttt{s} becomes \texttt{"hihi"}.
4. The function then returns a new string that starts with \texttt{"b"}, followed by \texttt{s} (now \texttt{"hihi"}), and ends with \texttt{"a"}.
5. The return value is therefore \texttt{"bhihia"}.
[/THOUGHT]

[ANSWER]
assert performOperation(s = "hi") == "bhihia"
[/ANSWER]
\end{tcolorbox}

\paragraph{AIME prompt.}
The same format as MATH-500 prompts, but without any in-context examples. 

\begin{tcolorbox}[
    colback=blue!5!white,      %
    colframe=blue!75!black,    %
    width=\linewidth,
    arc=2mm,                   %
    boxrule=1pt,               %
    left=3mm, right=3mm,       %
    top=2mm, bottom=2mm,
    title=\textbf{AIME prompt example}
]
\noindent\textbf{Problem: \{actual test question\}}

\medskip
\noindent\textbf{Solution:} 
\end{tcolorbox}

\paragraph{Output segmentation for PRMs.}
The output from our models follow a structured format shown below.
We segment each solution into individual reasoning steps using the marker \texttt{Step \#\#}, and feed the resulting step-level segments to the PRM \texttt{Qwen2.5-Math-PRM-7B} for evaluation. Prior work commonly uses \texttt{\textbackslash n\textbackslash n} as a step delimiter and finds it effective \citep{zhang2025lessonsdevelopingprocessreward}. In our setting, however, we find that \texttt{Step \#\#} better matches the formatting of our outputs and yields higher overall performance; we therefore adopt \texttt{Step \#\#} for segmentation throughout (see \cref{tab:space_result}).
\looseness=-1

\begin{tcolorbox}[
    colback=blue!5!white,      %
    colframe=blue!75!black,    %
    width=\linewidth,
    arc=2mm,                   %
    boxrule=1pt,               %
    left=3mm, right=3mm,       %
    top=2mm, bottom=2mm,
    title=\textbf{Expected generation from both Llama models about MATH questions}
    ]
    Step 1: [Description of first step]

Step 2: [Description of second step]

Step 3: [Description of third step]

Step...

The final answer is: \verb|$\boxed{}$|

\end{tcolorbox}

\begin{table}[!ht]
  \centering
  \caption{Effects of segmentation keywords on MATH-500 with \texttt{Llama-3.1-8B-Instruct}}
  \label{tab:space_result}
  \begin{tabular}{@{}lcccc@{}}
    \toprule
    \textbf{Method} & \textbf{4} & \textbf{8} & \textbf{16} & \textbf{32} \\
    \midrule  
     Uniform (delimiter: \#\#) & 60.26\% & 63.30\% & 66.56\% & 67.80\% \\
     Uniform (delimiter: \texttt{\textbackslash n\textbackslash n})& 59.06\% & 61.50\% & 64.36\% & 64.76\% \\
     Elim (delimiter: \#\#)& 62.50\% & 65.90\% & 67.96\% & 69.46\% \\
     Elim (delimiter: \texttt{\textbackslash n\textbackslash n})         & 61.96\% & 64.66\% & 65.66\% & 66.20\% \\
    \bottomrule
  \end{tabular}
\end{table}

\subsection{Additional experimental results}
\label{app:additional_experiments}

\subsubsection{Additional experiments on MATH-500 and LiveCodeBench}
\label{app:additional_experiments}

\paragraph{Results of extended exploration rules on MATH-500.}
\cref{fig:coverage_all} provides the coverage performance of \ucbtext and \gaptext on the MATH-500 dataset; the \texttt{max\_samples} hyperparameter for both \ucbtext and \gaptext are provided in \cref{tab:max-samples-expanded}.

\begin{figure}[ht]
  \centering
  \includegraphics[width=.45\linewidth]{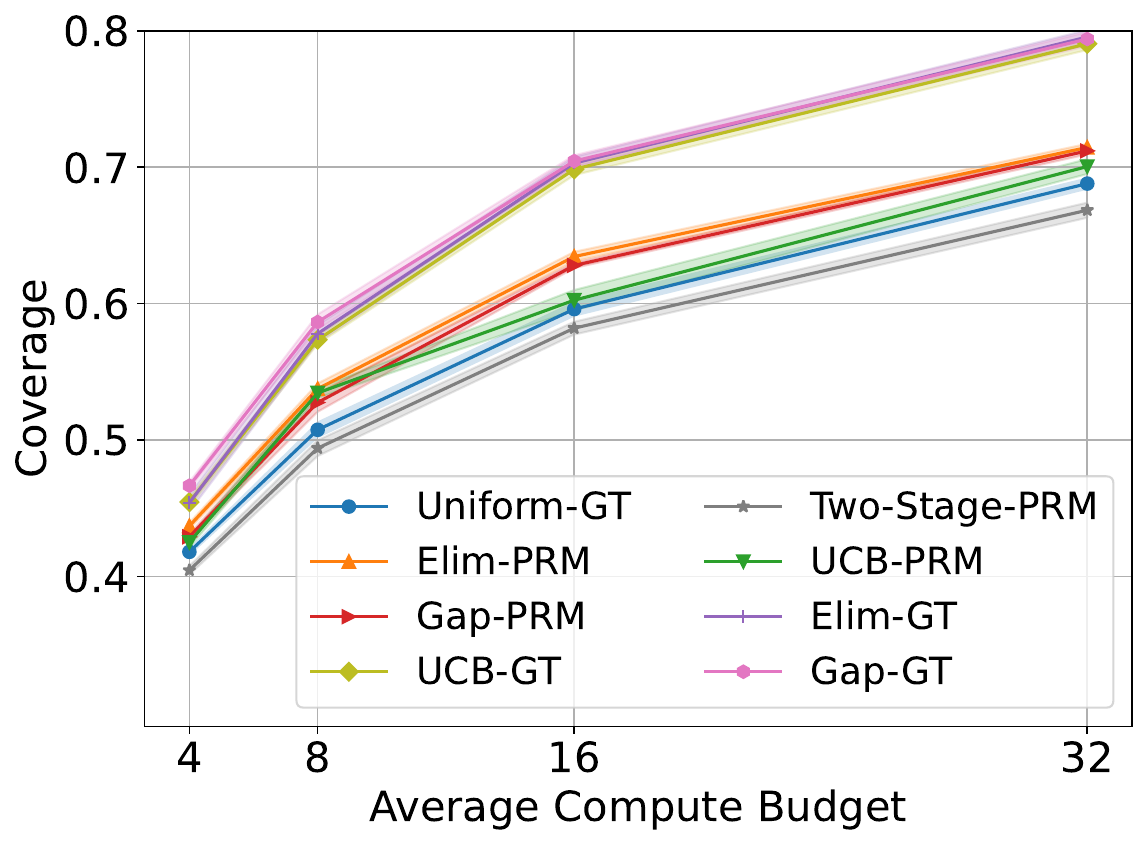}\hfill
  \includegraphics[width=.45\linewidth]{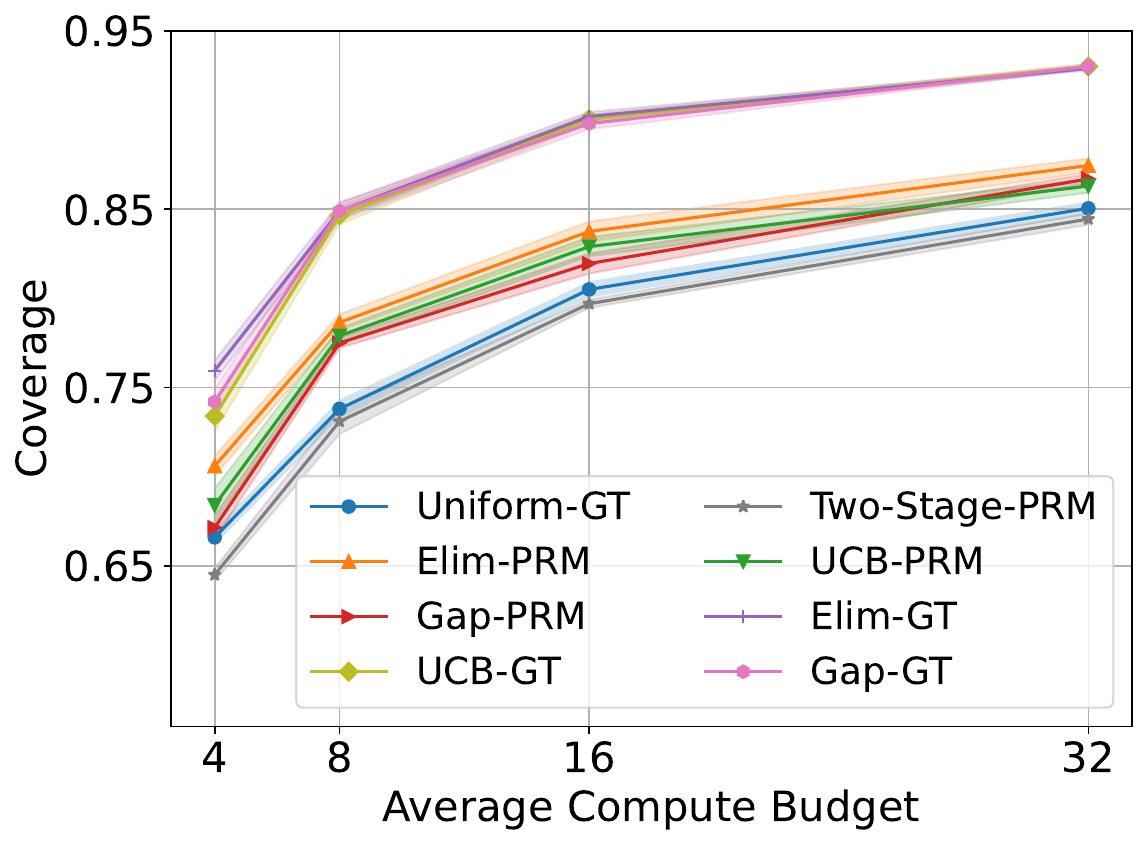}
  \caption{Results on MATH-500 with extended exploration rules. \emph{Left:} Results with \texttt{Llama-3.2-1B-Instruct}. \emph{Right:} Results with \texttt{Llama-3.1-8B-Instruct}.}
  \label{fig:coverage_all}
  \vspace{-10 pt}
\end{figure}

\begin{table}[ht]
  \centering
    \caption{The choice of \texttt{max\_samples} by model-scene and compute budget on MATH-500.}
  \label{tab:max-samples-expanded}
  \begin{tabular}{@{}lcccc@{}}
    \toprule
    \textbf{Scene} & \textbf{Avg. budget 4} & \textbf{Avg. budget 8} & \textbf{Avg. budget 16} & \textbf{Avg. budget 32} \\
                   & (\texttt{max\_samples}) & (\texttt{max\_samples}) & (\texttt{max\_samples}) & (\texttt{max\_samples}) \\
    \midrule
    \texttt{Llama-3.1} w/ GT    & 40                & 40                 & 120                  & 300                   \\
    \texttt{Llama-3.1} w/ PRM   & 12                &  40                & 80                  & 120                   \\
    \texttt{Llama-3.2} w/ GT    & 40                & 40                 & 120                  & 120                   \\
    \texttt{Llama-3.2} w/ PRM   & 12                & 12                 & 60                   & 60                    \\
    \bottomrule
  \end{tabular}

\end{table}

\paragraph{Online compute allocation on MATH-500.}
We further evaluate a streaming variant of \cref{alg:methods} in \cref{tab:streaming}, where queries arrive sequentially. In this setting, compute is allocated on a per-query basis over 500 rounds, rather than jointly across the entire dataset.
For the streaming variant, we tune the \texttt{max\_samples} hyperparameter via a binary search over the grid $\{15 + 5k \mid k = 0, \dots, 37\}$.
Even under the streaming constraint, our algorithm still outperforms uniform allocation with pool access.
\looseness=-1

\begin{table}[H]
  \centering
  \caption{Results of the streaming variant of \cref{alg:methods} on MATH-500 with \texttt{Llama-3.2-1B-Instruct}.}
  \label{tab:streaming}
  \begin{tabular}{@{}lcccc@{}}
    \toprule
    & \multicolumn{4}{c}{\textbf{Average compute budget}} \\
    \cmidrule(l){2-5}
    \textbf{Method} & \textbf{4} & \textbf{8} & \textbf{16} & \textbf{32} \\
    \midrule
    \multicolumn{5}{@{}l}{\textbf{Coverage}} \\
    \addlinespace[1pt]
    Uniform-PRM      & 41.80\% & 50.75\% & 59.60\% & 68.80\% \\
    Elim-PRM         & 43.70\% & 53.75\% & 63.45\% & 71.45\% \\
    Elim-PRM (Streaming)       & 42.65\% & 52.85\% & 62.95\% & 71.05\% \\
    \addlinespace
    \multicolumn{5}{@{}l}{\textbf{Accuracy}} \\
    \addlinespace[1pt]
    Uniform-PRM      & 35.00\% & 40.70\% & 46.45\% & 49.45\% \\
    Elim-PRM         & 37.15\% & 43.70\% & 48.90\% & 52.60\% \\
    Elim-PRM (Streaming)       & 36.55\% & 43.10\% & 48.55\% & 52.20\% \\
    \bottomrule
  \end{tabular}
\end{table}

\paragraph{Token-controlled evaluation.}
To isolate the effect of allocation from differences in total generation length, we run a token-controlled variant of \cref{alg:methods}. Specifically, we first record the total number of generated tokens under \textsc{Uniform} allocation, and then run \cref{alg:methods} while enforcing the same total token budget (token-matched). As shown in \cref{tab:math500-llama32-1b-token} (MATH-500) and \cref{tab:lcb-deepseek-token} (LiveCodeBench), \cref{alg:methods} remains consistently better than under \textsc{Uniform} with the same number of tokens.

\begin{table}[ht]
  \centering
  \caption{Token-matched evaluation on MATH-500 with \texttt{Llama-3.2-1B-Instruct}.}
  \label{tab:math500-llama32-1b-token}
  \begin{tabular}{lcccc}
    \toprule
    & \multicolumn{4}{c}{\textbf{Average compute budget}} \\
    \cmidrule(l){2-5}
    \textbf{Method} & \textbf{4} & \textbf{8} & \textbf{16} & \textbf{32} \\
    \midrule
    \multicolumn{5}{l}{\textbf{Coverage}} \\
    \addlinespace[1pt]
    Uniform-PRM      & 41.80\% & 50.75\% & 59.60\% & 68.80\% \\
    Elim-PRM         & 43.70\% & 53.75\% & 63.45\% & 71.45\% \\
    Elim-PRM (Token Matched)        & 42.70\% & 52.10\% & 60.90\% & 69.95\% \\
    \addlinespace
    \multicolumn{5}{l}{\textbf{Accuracy}} \\
    \addlinespace[1pt]
    Uniform-PRM      & 35.00\% & 40.70\% & 46.45\% & 49.45\% \\
    Elim-PRM         & 37.15\% & 43.70\% & 48.90\% & 52.60\% \\
    Elim-PRM (Token Matched)        & 36.50\% & 42.60\% & 48.05\% & 51.85\% \\
    \bottomrule
  \end{tabular}
\end{table}

\begin{table}[!ht]
  \centering
  \caption{Token-matched evaluation on LiveCodeBench with \texttt{DeepSeek-R1-Distill-Llama-8B}.}
  \label{tab:lcb-deepseek-token}
  \begin{tabular}{lcccc}
    \toprule
    & \multicolumn{4}{c}{\textbf{Average compute budget}} \\
    \cmidrule(l){2-5}
    \textbf{Method} & \textbf{4} & \textbf{8} & \textbf{16} & \textbf{32} \\
    \midrule
    \multicolumn{5}{l}{\textbf{Coverage}} \\
    \addlinespace[1pt]
    Uniform-GT      & 73.43\% & 81.05\% & 86.01\% & 89.87\% \\
    Elim-GT         & 84.66\% & 91.02\% & 94.00\% & 94.00\% \\
    Elim-GT (Token Matched)        & 84.50\% & 91.65\% & 94.00\% & 94.00\% \\
    \bottomrule
  \end{tabular}
\end{table}

\begin{figure}[ht]
  \centering
  \includegraphics[width=.45\linewidth]{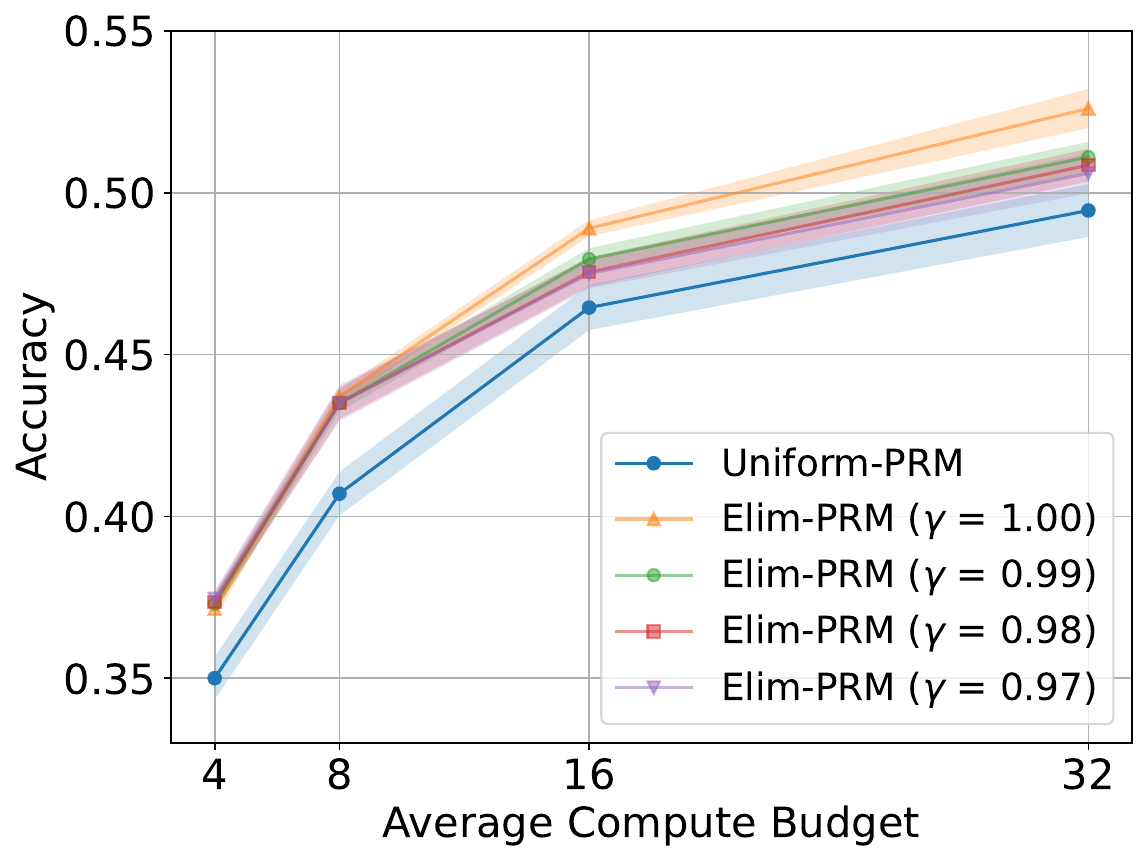}
  \caption{Ablation study for the hyperparameter $\gamma$. Experiments conducted on MATH-500 with \texttt{Llama-3.2-1B-Instruct}.}
\label{fig:thresholds_llama32}
\end{figure}

\paragraph{Choices of the threshold $\gamma$.} 
\label{app:thresholds}
The elimination threshold $\gamma$ controls which queries are confidently answered and can be removed from the active set.  
Since higher reward scores generally indicate higher-quality responses, setting a high threshold $\gamma \in [0, 1]$ is natural.  
We conduct ablations with $\gamma \in \{0.97, 0.98, 0.99, 1.0\}$ on MATH-500 and report the results in \cref{fig:thresholds_llama32}.  
We observe that $\gamma = 1.0$ performs slightly better, likely because \texttt{Qwen2.5-Math-PRM-7B} assigns a deterministic score of 1.0 to answers it deems correct—a property specific to this PRM.  
Importantly, across all tested values, our method consistently outperforms the uniform allocation baseline, indicating that \cref{alg:methods} is robust to variations in $\gamma$.

\paragraph{Effect of PRM quality.}
To study the impact of PRMs quality on the effectiveness of our method, we conducted additional experiments using a relatively weak PRM, \texttt{Qwen2.5-Math-7B-PRM800K}, in contrast to the strong PRM, \texttt{Qwen2.5-Math-PRM-7B}, employed in the main experiments. As shown in the \cref{fig:weakprm}, our method consistently outperforms the baselines across all compute budgets under both weak and strong reward models. Notably, performance improves with a stronger oracle, suggesting that our algorithm benefits directly from higher-quality reward signals.

\begin{figure}[ht]
  \centering
  \begin{subfigure}[t]{0.45\textwidth}
    \centering
    \includegraphics[width=\linewidth]{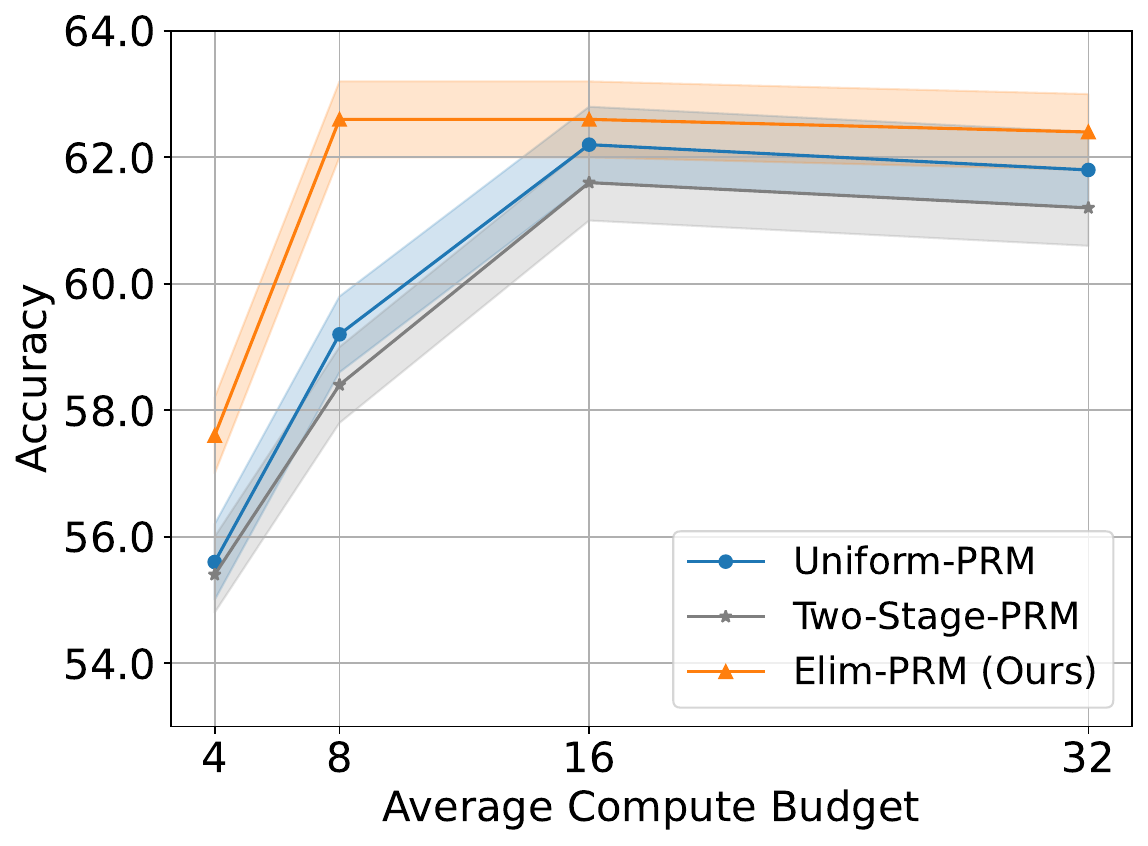}
     \caption{Weak PRM: \texttt{Qwen2.5-Math-7B-PRM800K}}
  \end{subfigure}
  \hfill
  \begin{subfigure}[t]{0.45\textwidth}
    \centering
    \includegraphics[width=\linewidth]{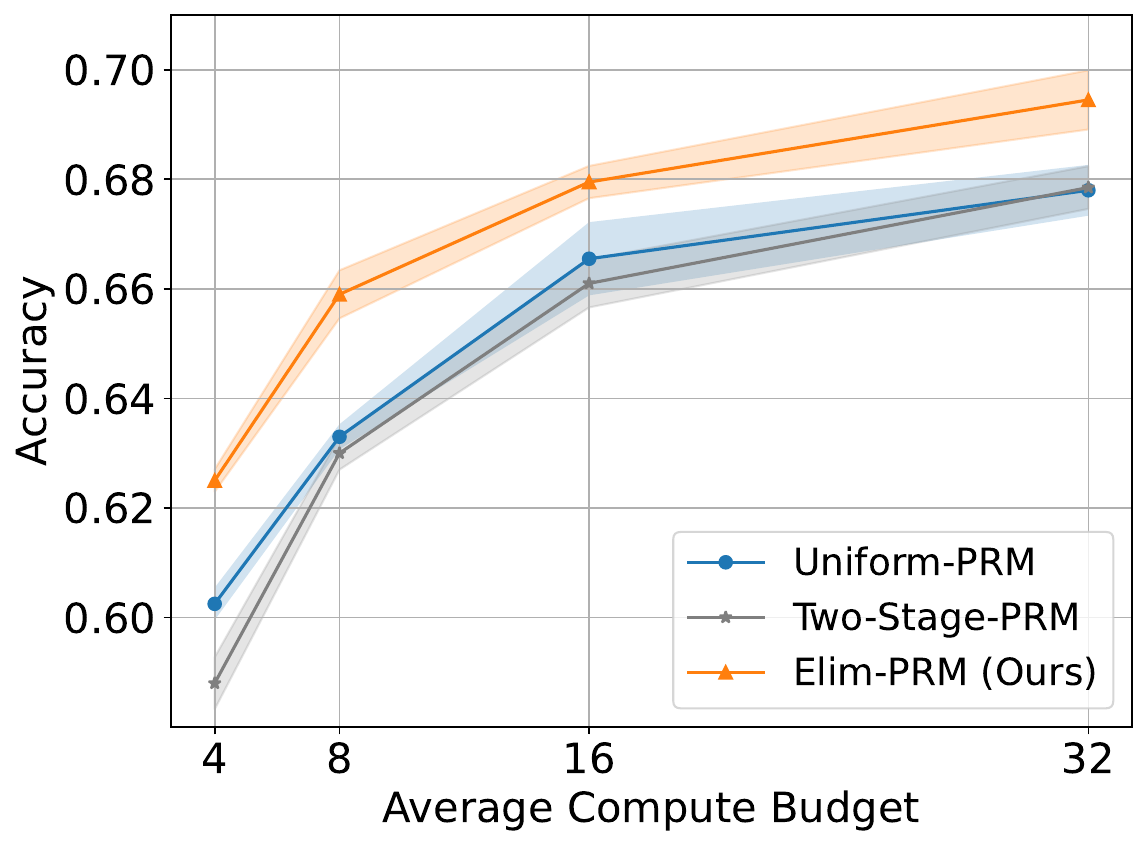}
     \caption{Strong PRM: \texttt{Qwen2.5-Math-PRM-7B}}
  \end{subfigure}

  \caption{Effect of PRM quality on MATH-500 with \texttt{Llama-3.1-8B-Instruct}. We compare performance using a weaker PRM (\texttt{Qwen2.5-Math-7B-PRM800K}) versus a stronger PRM (\texttt{Qwen2.5-Math-PRM-7B}) across compute budgets.
  }
  \vspace{-10pt}
  \label{fig:weakprm}
\end{figure}

\subsubsection{Additional experiments and analyses on MATH-500-Hard}
\label{app:math_hard}

\begin{figure}[ht]
  \centering
  \begin{subfigure}[t]{0.45\textwidth}
    \centering
    \includegraphics[width=\linewidth]{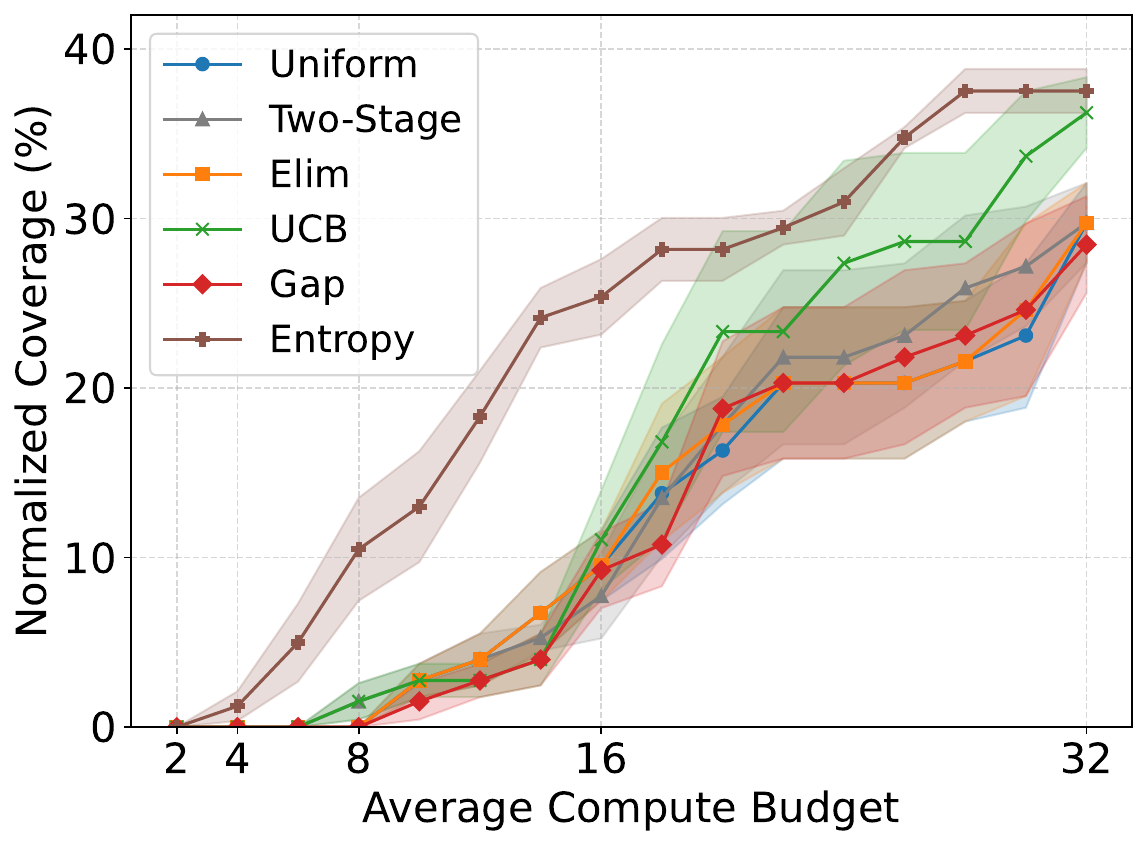}
    \caption{Coverage with \texttt{Llama-3.2-1B-Instruct} on MATH‑500‑Hard‑8}
    \label{fig:model2}
  \end{subfigure}
  \hfill
    \begin{subfigure}[t]{0.45\textwidth}
    \centering
    \includegraphics[width=\linewidth]{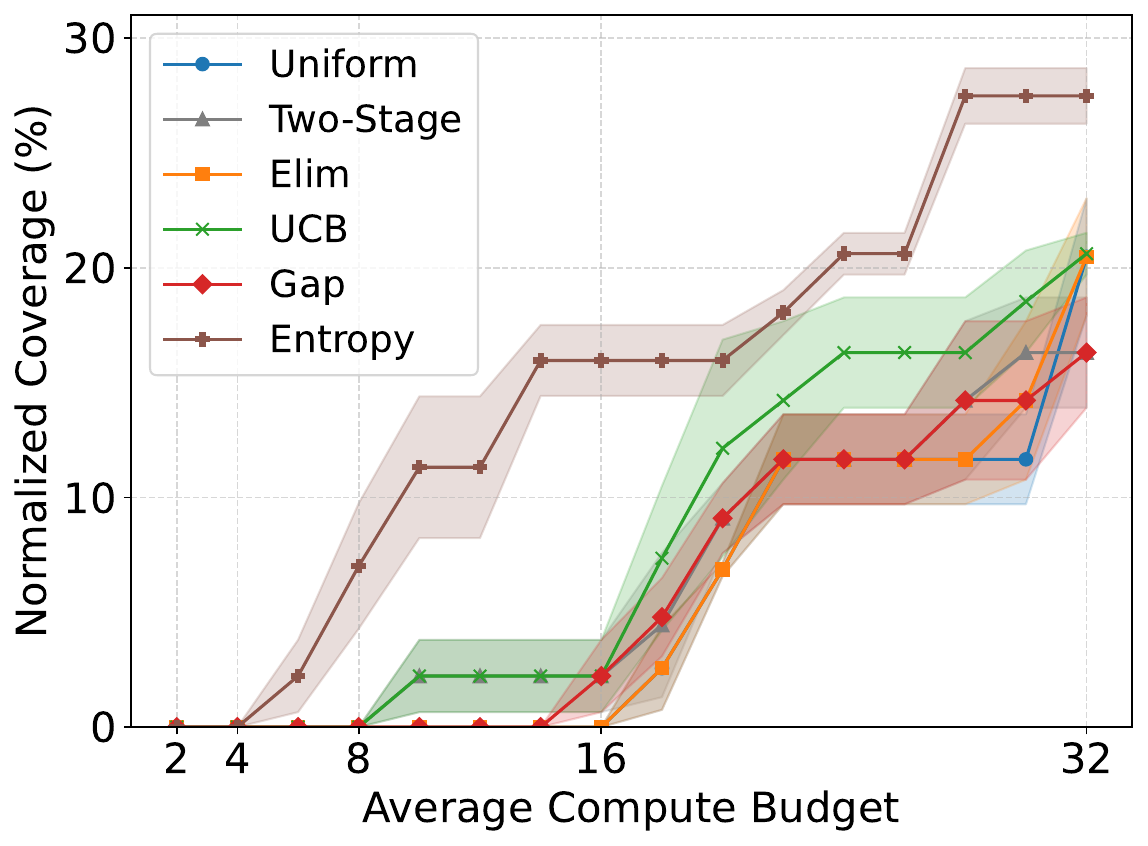}
    \caption{Coverage with \texttt{Llama-3.2-1B-Instruct} on MATH‑500‑Hard‑16}
    \label{fig:model4}
  \end{subfigure}

  \vspace{0.8em}

  \begin{subfigure}[t]{0.45\textwidth}
    \centering
    \includegraphics[width=\linewidth]{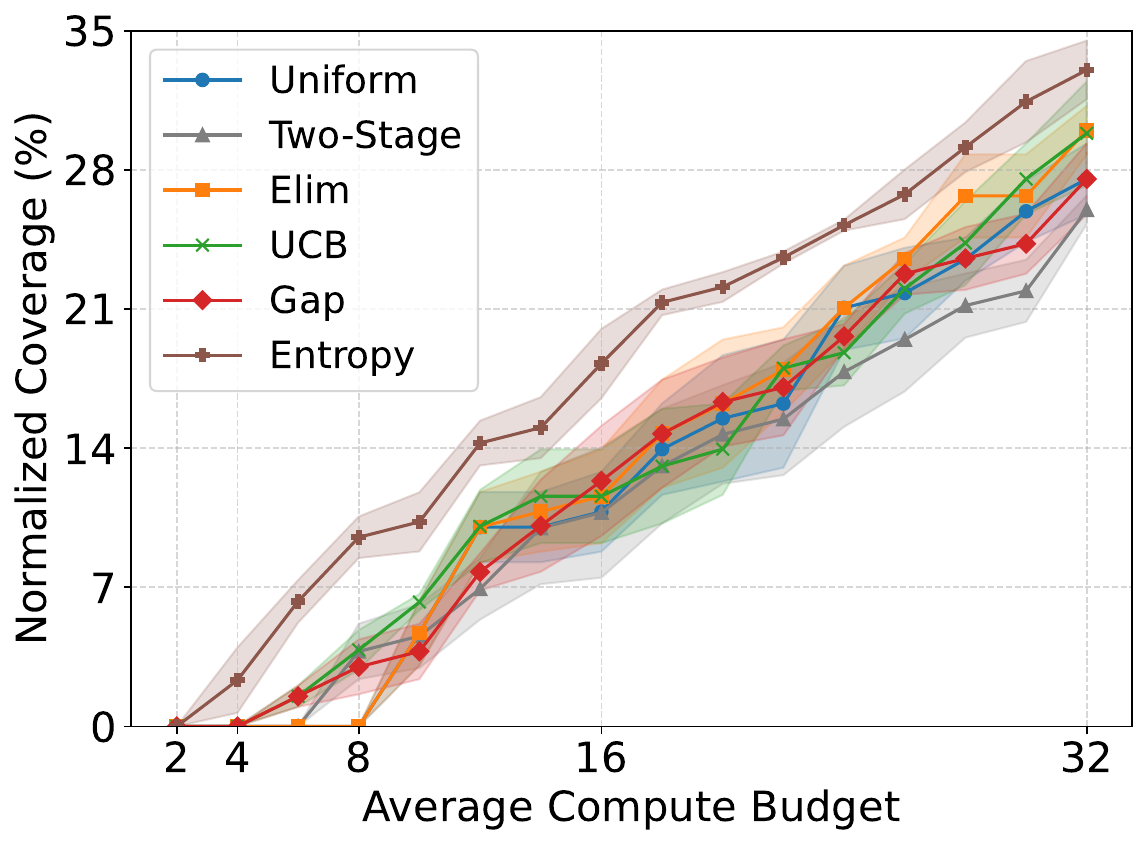}
    \caption{Coverage with \texttt{Llama-3.1-8B-Instruct} on MATH‑500‑Hard‑8}
    \label{fig:model1}
  \end{subfigure}
  \hfill
  \begin{subfigure}[t]{0.45\textwidth}
    \centering
    \includegraphics[width=\linewidth]{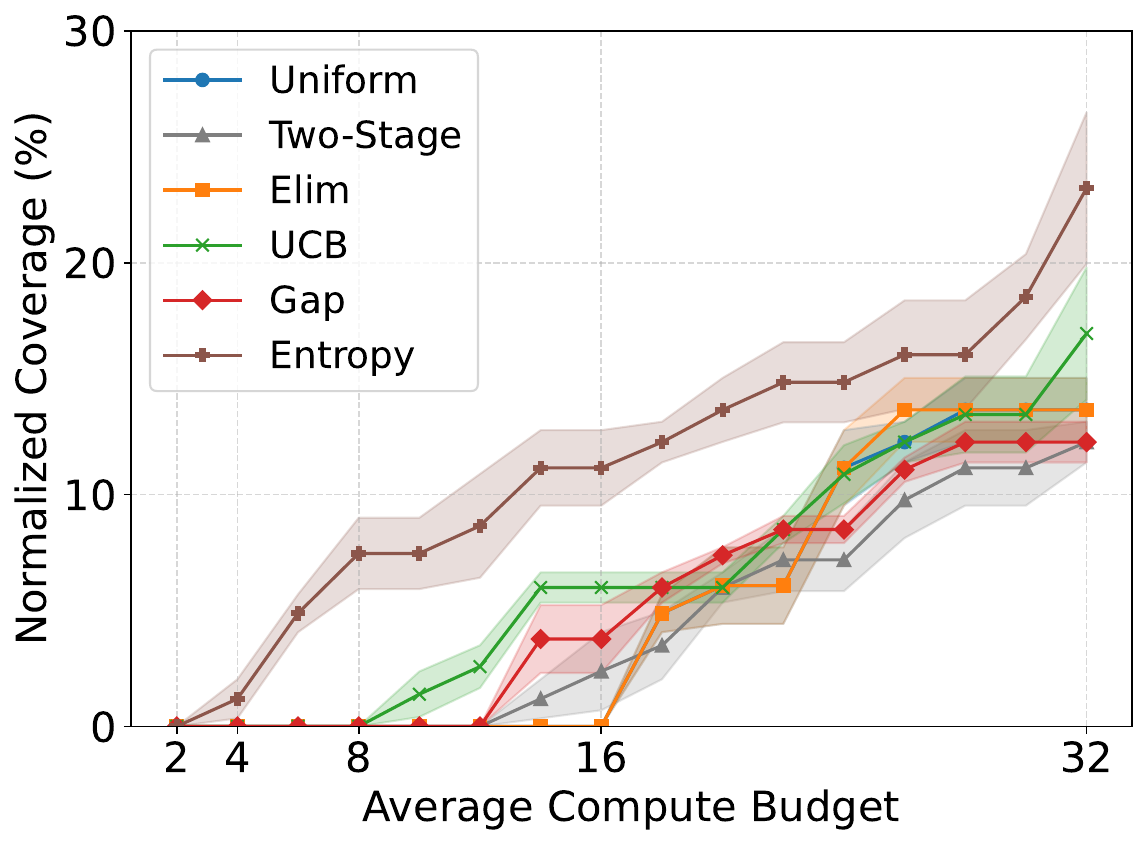}
    \caption{Coverage with \texttt{Llama-3.1-8B-Instruct} on MATH‑500‑Hard‑16}
    \label{fig:model3}
  \end{subfigure}

  \caption{Coverage comparisons on MATH-500-Hard datasets with two language model of different sizes: \texttt{Llama-3.2-1B-Instruct} and \texttt{Llama-3.1-8B-Instruct}.}
  \label{fig:models_results}
  \vspace{-10pt}
\end{figure}

\cref{fig:models_results} presents additional results on MATH-500-Hard datasets, including experiments with the \textsc{Gap} algorithm and experiments with the \texttt{Llama-3.1-8B-Instruct} model. These experiments show that our algorithms introduced in \cref{sec:extensions}, particularly \entropytext{} and \textsc{UCB}, achieve early advantages and outperform both \textsc{Uniform} and \textsc{Elimination}. 

\paragraph{Why \entropytext works well on challenging datasets?}
We conduct a detailed analysis on understand why \entropytext performs particularly well on MATH-500-Hard datasets. 
Upon inspecting model responses to challenging queries, we observe that unsolvable queries are more likely to yield invalid outputs (e.g., incomplete or improperly formatted), resulting in lower entropy among their generated responses; in contrast, solvable queries tend to generate more diverse outputs (see \cref{tab:grouped-stats} and \cref{tab:grouped-stats2} for statistics computed from 64 responses per query).
Since \textsc{Entropy} prioritizes queries with higher entropy, it naturally allocates more compute to those that are more likely to be solvable---explaining its strong empirical performance on challenging problems.
We expect this behavior to generalize to other challenging benchmarks, provided that invalid responses can be reliably identified.  
In such settings, \entropytext offers an effective means to shift compute toward promising queries and achieve better performance under limited compute budget.

\begin{table}[ht]
  \centering
  \caption{Aggregated statistics by query group on MATH-500-Hard-8.}
  \label{tab:grouped-stats}
  \begin{tabular}{@{}lccc@{}}
    \toprule
    \textbf{Query group} & \textbf{$\#$questions} & \textbf{Entropy (mean)} & \textbf{Invalid answers (\%)} \\
    \midrule
    Unsolvable & 49 & 4.26 & 19.45\% \\
    Solvable & 22 & 4.52 & 12.45\% \\
    \bottomrule
  \end{tabular}
\end{table}

\begin{table}[H]
  \centering
  \caption{Aggregated statistics by query group on MATH-500-Hard-16.}
  \label{tab:grouped-stats2}
  \begin{tabular}{@{}lccc@{}}
    \toprule
    \textbf{Query group} & \textbf{$\#$questions} & \textbf{Entropy (mean)} & \textbf{Invalid answers (\%)} \\
    \midrule
    Unsolvable & 43 & 4.33 & 18.39\% \\
    Solvable & 13 & 4.53 & 13.10\%\\
    \bottomrule
  \end{tabular}
\end{table}

\paragraph{Hyperparameter study of $\lambda$ for \entropytext and \ucbtext.}
We study the effect of the exploration coefficient $\lambda$ in our Entropy- and UCB-based extension rules on the MATH-500-Hard-8 dataset. Results in \cref{tab:entropy_alpha,tab:ucb_alpha}
suggest that \entropytext benefits from larger $\lambda$, while UCB prefers smaller $\lambda$.

\begin{table}[H]
\centering
\caption{Effect of $\lambda$ on Entropy-based Allocation.}
\label{tab:entropy_alpha}
\begin{tabular}{lcccc}
\toprule
\textbf{$\lambda$ choice / Avg compute budget} & \textbf{4} & \textbf{8} & \textbf{16} & \textbf{32} \\
\midrule
Entropy ($\lambda = 1.0$) & 0.0\% & 5.4\% & 8.9\% & 17.9\% \\
Entropy ($\lambda = 2.0$) & 0.0\% & 5.4\% & 8.9\% & 17.9\% \\
Entropy ($\lambda = 3.0$) & 0.0\% & 5.4\% & 8.9\% & 21.4\% \\
\bottomrule
\end{tabular}
\end{table}

\begin{table}[H]
\centering
\caption{Effect of $\lambda$ on UCB-based Allocation.}
\label{tab:ucb_alpha}
\begin{tabular}{lcccc}
\toprule
\textbf{$\lambda$ choice / Avg compute budget} & \textbf{4} & \textbf{8} & \textbf{16} & \textbf{32} \\
\midrule
UCB ($\lambda = 1.0$) & 0.0\% & 1.8\% & 3.6\% & 16.1\% \\
UCB ($\lambda = 2.0$) & 0.0\% & 1.8\% & 3.6\% & 14.3\% \\
UCB ($\lambda = 3.0$) & 0.0\% & 1.8\% & 1.8\% & 12.5\% \\
\bottomrule
\end{tabular}
\end{table}

\end{document}